\documentclass[12pt,a4paper]{article}

\usepackage{amsthm,amsmath,multirow,graphicx,subfigure,booktabs,url,mathtools,wrapfig,mathrsfs,bm,relsize,caption,xcolor, enumitem, fullpage,float,bbm}

\RequirePackage[authoryear]{natbib}
\RequirePackage[psamsfonts]{amssymb}
\usepackage[colorlinks=true,
linkcolor=blue,
urlcolor=blue,
citecolor=blue]{hyperref}

\numberwithin{equation}{section}
\newtheorem{conj}{Conjecture}
\newtheorem{thm}[conj]{Theorem}
\newtheorem{cor}[conj]{Corollary}
\newtheorem{prop}[conj]{Proposition}
\newtheorem{lemma}[conj]{Lemma}

\newtheorem{innercustomgeneric}{\customgenericname}
\providecommand{\customgenericname}{}
\newcommand{\newcustomtheorem}[2]{%
	\newenvironment{#1}[1] 
	{%
		\renewcommand\customgenericname{#2}%
		\renewcommand\theinnercustomgeneric{##1}%
		\innercustomgeneric
	}
	{\endinnercustomgeneric}
}

\newcustomtheorem{customAss}{Assumption}

\theoremstyle{remark}\newtheorem{remark}{Remark}

\def\Cov{\text{Cov}}   
\def\PP{\mathbb{P}}
\def\EE{\mathbb{E}}
\def\RR{\mathbb{R}}
\def\wh{\widehat}
\def\eps{\varepsilon}
\def\bI{\bm{I}}
\def\wt{\widetilde}

\def\a{\alpha}
\def\T{\top}
\def\i{\infty}
\def\sz{\Sigma_Z}

\def\szy{\Sigma_{Z|Y}}

\def\C{\szy}
\def\sw{\Sigma_W}
\def\errw{\delta_W}
\def\rank{{\rm rank}}
\def\tr{{\rm tr}}
\def\op{{\rm op}}
\def\diag{\textrm{diag}}

\def\sgn{{\rm sign}}

\def\rI{\textrm{I}}
\def\rII{\textrm{II}}

\def\Dt{\Delta}

\def\cO{\mathcal O}
\def\cS{\mathcal {S}}
\def\cN{\mathcal {N}}
\def\0{\bm 0}
\def\1{\mathbbm 1}
\def\b1{\bm 1}
\def\beps{\bm \eps}

\def\cO{\mathcal{O}}
\def\cE{\mathcal{E}} 
\def\cS{\mathcal{S}}
\def\cN{\mathcal{N}}

\def\bX{\mathbf{X}}
\def\by{\mathbf{y}}
\def\bY{\mathbf{y}}
\def\bZ{\mathbf{Z}}
\def\bW{\mathbf{W}}

\def\bI{\mathbf{I}}
\def\bG{\mathbf{G}}

\def\PP{\mathbb{P}}
\def\EE{\mathbb{E}}
\def\RR{\mathbb{R}}

\def\rank{{\rm rank}}
\def\tr{{\rm tr}}
\def\op{{\rm op}}
\def\diag{{\rm diag}}
\def\rI{{\rm I}}
\def\rII{{\rm II}}

\def\wh{\widehat}
\def\wt{\widetilde}

\def\eps{\varepsilon}
\def\sw{\Sigma_W}
\def\sz{\Sigma_Z}
\def\szy{\Sigma_{Z|Y}}

\def\Cov{\text{Cov}}   
\def\errw{\delta_W}
\def \Dt{\Delta}
\def\i{\infty}
\def\T{\top}

\title{Interpolating Discriminant Functions in High-Dimensional Gaussian Latent Mixtures}
\author{Xin Bing\thanks{Department of 
		Statistical Sciences, University of Toronto. E-mail: \texttt{xin.bing@toronto.ca}}
~~~~~Marten Wegkamp\thanks{Department of Mathematics and   Department of Statistics and Data Science, Cornell University.  E-mail: \texttt{marten.wegkamp@cornell.edu}.
} }
\date{}

\begin{document}
	
	\maketitle
	
	\begin{abstract}
		This paper considers binary  classification of high-dimensional features under a postulated model with a low-dimensional latent Gaussian mixture structure and non-vanishing noise. 
		A generalized least squares estimator is used to estimate the direction of the optimal separating hyperplane. The estimated hyperplane is shown to interpolate on the training data. While    the direction vector can be consistently estimated as could be expected from recent results in linear regression, a naive plug-in estimate fails to consistently
		estimate the intercept.  A simple correction, that requires an independent hold-out sample, renders the procedure minimax optimal in many scenarios. The interpolation property of the latter procedure can be retained, but surprisingly depends on the way the labels are encoded.
		
	\end{abstract}
	{\em Keywords:}  High-dimensional classification, latent factor model, generalized least squares, interpolation, benign overfitting, overparametrization, discriminant analysis, minimax optimal rate of convergence.

	\section{Introduction}
	We consider binary classification of a high-dimensional feature vector. That is, we are given a $n\times p$ data-matrix $\bX$ with $n$ independent $p$-dimensional rows, with $p\gg n$, and a response vector $\by\in \{0,1\}^n$ of corresponding labels, and the task is to predict the labels of  new $p$-dimensional features. Complex models such as  kernel support vector machines (SVM) and deep neural networks have been observed to have surprisingly good generalization performance despite overfitting the training data. Towards understanding such benign overfitting phenomenon, one popular example  that has  gained increasing attention  is the 
	estimator $\wh \theta = \bX^+ \by$
	in the regression context.
	Here $\bX^{+}$ represents the Moore-Penrose inverse of $\bX$. 
	It can be shown that $\wh\theta$  has the minimal $\ell_2$-norm among all  least squares  estimators, that is,\[
	\|\wh\theta\|_2 = \min \left\{ \|\theta\|_2: \ \| \by - \bX \theta\|_2=\min_{u\in \RR^p} \| \by-\bX u\|_2\right\}.\]  Since often $\bX\wh \theta =\by$ holds, for instance, when $\bX$ has full row-rank $n< p$, the estimator $\wh \theta$ is referred to as the minimum-norm interpolator.
	It serves as 
	the prime example to illustrate the phenomenon that overfitting 
	in linear regression can be benign, in that $\widehat\theta$ can still lead to good prediction results.
	See, for instance, \cite{belkin2018overfitting,bartlett2019,hastie2019surprises,BSW22} and the references therein.

	More recently, the minimum-norm interpolator  further finds its importance in binary classification problems. Specifically, $\wh \theta$ is shown to coincide with the solution of the hard margin SVM under the over-parametrized Gaussian mixture models \citep{muthukumar2019harmless,wang2021benign} and beyond \citep{hsu2021proliferation}. In the over-parametrized logistic regression model, $\wh \theta$ is also closely connected to the solution of maximizing the log-likelihood, obtained by gradient descent with sufficiently small step size \citep{soudry2018implicit,cao2021risk}. 
	In the over-parametrized setting $p\gg n$, 
	the hyperplane $\{ x \mid x^\T \wh \theta =0\}$ separates the training data perfectly, leading to a classifier that has zero training error. There is a growing  literature \citep{cao2021risk,wang2021benign,chatterji2021finite,minsker2021minimax} that shows that interpolating classifiers 
	$\bar g(x)=\1\{ x^\T \wh\theta >0\}$ can also have vanishing misclassification error $\PP\{ \bar g(X)\ne Y\}$
	in (sub-)Gaussian mixture models. We extend these works in this paper motivated by the following observations. 
	
	First, we notice that the separating hyperplane $\{x \mid  x^\T \wh \theta =0\}$ considered in the above mentioned literature  has no intercept. 
	Under symmetric  Gaussian mixture models, that is, $\PP\{Y = 1\} = \PP\{Y = 0\}$ and $X \mid Y = k \sim N_p((2k-1)\mu, \Sigma)$ for $k\in \{0,1\}$, the optimal (Bayes) rule is indeed based on a hyperplane through the origin (no intercept).  However, this is no longer true in the asymmetric setting where  the class probabilities differ,
	$\PP\{Y = 1\} \ne \PP\{Y = 0\}$,
	rendering the usage of $\bar g(x)$ questionable. Although \cite{wang2021benign} shows that in the asymmetric setting $\PP\{ \bar g(X)\ne Y\}$ still tends to zero if the separation between $X \mid Y = 0$  and $X \mid Y = 1$ diverges, the rate of this convergence is unfortunately exponentially slower than the optimal rate, in part due to not using an intercept. This motivates us to propose an improved linear classifier 
	based on $\wh\theta$ that includes an intercept, formally introduced
	in (\ref{def_g_hat_intro}) below. Finding a meaningful intercept under the interpolation of $\wh\theta$ requires extra care, as standard approaches, such as the empirical risk minimization, can not be used when the hyperplane $\{x \mid x^\T \wh \theta=0\}$ separates the training data perfectly.

	Second,  the aforementioned works all focus on the misclassification risk $\PP\{ \bar g(X)\ne Y\}$. In particular, 
	they show that $\PP\{ \bar g(X)\ne Y\}$ vanishes only if the separation between the two mixture distributions diverges. 
	In general, the  {\em excess} risk - the difference between the misclassification error and the Bayes error - 
	is   a more meaningful 
	criterion,
	because the Bayes error $\inf_h \PP\{h(X)\ne Y\} $  is the smallest possible misclassification error among all classifiers and generally does not vanish.
For this reason, we  focus on analyzing the excess risk of the proposed classifier, and our results are informative even if the separation between the two mixture distributions does not diverge (and therefore the Bayes risk does not vanish).

Summarizing, the existing results on interpolating classifiers  are not satisfactory as 
they only consider stylized examples that do not address the more realistic scenarios  
when the mixture probabilities  are asymmetric and the Bayes error does not vanish.
In fact, we will argue that  these interpolation methods without intercept in the literature actually fail in the asymmetric setting when the conditional distributions are not asymptotically distinguishable - which is the statistically more challenging case.
In this work, we instead analyze the interpolating classifier with a judiciously chosen 
intercept under a recently proposed latent, low-dimensional statistical model \citep{BW22}, and our results
reveal that its excess risk  has minimax-optimal rate of convergence in the over-parametrized setting even when the separation between the two mixture distributions does not diverge. 
Together with the interpolation property of the proposed  classifier, we thus provide a concrete instance of the interesting  phenomenon that overfitting and minimax-optimal generalization performance can coexist in a more realistic statistical  setting, against traditional statistical belief.

\subsection{Our contributions} 

Concretely, in this paper we study the linear classifier 
\begin{equation}\label{def_g_hat_intro}
	\wt g(x)  = \1\{x^\T \wh \theta + \wt \beta_0>0\}, \qquad\text{for any } x\in\RR^p,
\end{equation}
based on the minimum-norm interpolator $\wh \theta$ and
some estimated intercept  $\wt \beta_0\in \RR$.
Following \cite{BW22}, we assume that each feature consists of linear combinations of hidden low-dimensional Gaussian components, obscured by independent, possibly non-Gaussian, noise. 
The low-dimensional Gaussian component suggests to take a Linear Discriminant Analysis (LDA) approach, reducing the problem  to find (i) the unknown low-dimensional space,   (ii) its dimension $K$, with $K\ll n$, and (iii) the optimal hyperplane $\{z\in \RR^K \mid z^\T \beta + \beta_0 = 0\}$ in this latent space. The formal model and expression of the optimal hyperplane are described in Section \ref{sec_background}.
Existing literature on
LDA  in high-dimensional classification problems often  imposes sparsity on the   coefficients of the hyperplane  \citep{Tibshirani2002,FanFan2008,Witten2011,Shao2011,CaiLiu2011,mai2012,caizhang2019}. 
In this work, we take a different route and do not   assume that  the high-dimensional features are Gaussian, nor rely on any sparsity assumption.

The recent work \cite{BW22} successfully utilized
Principal Component Regression (PCR) to estimate $z^\T \beta$ with its low-dimension estimated via the method developed in \cite{rank19}. 
The classifier in \eqref{def_g_hat_intro} estimates $z^\T \beta$ by $x^\T \wh\theta$ via  the minimum-norm interpolator $\wh \theta$ instead. This estimator can be viewed as a limit case of PCR, with the number of retained principal components equal to $p$ (not $K$), and as such an extension of \cite{BW22}. 
The practical advantage of this extension  is to avoid estimation of the latent dimension $K$, meanwhile it sheds light on the robustness of the PCR-based classifiers of \cite{BW22} against misspecification of $K$. From a theoretical perspective, it is surprising to see that $x^\T\wh\theta$ adapts to the low-dimensional structure in $z^\T\beta$, as explained below.

Section \ref{sec_rates} provides theoretical guarantees for our proposed classifier. 
Theorem \ref{thm_rate_Atheta} in Section \ref{sec_rates} states that $x^\T\wh\theta$ consistently estimates $z^\T\beta$ by adapting to the low-dimensional structure, and the rate of this convergence is often minimax-optimal in over-parametrized setting. Establishing Theorem \ref{thm_rate_Atheta} is our main technical challenge and its proof occupies a large part of the paper and is delegated to Section \ref{sec:proof:main}.
Although this convergence is in line with current developments in regression that $\wh \theta$ surprisingly succeeds, our analysis is more complicated as $\bY$ is no longer linearly related with $\bX$. In Section \ref{sec_technical} we also explain that 
similar arguments explored in \cite{BW22} can not be used here. A key step in our proof is to recognize and characterize the implicit regularization of $\wh\theta = \bX^+\bY$ in high-dimensional factor models. In our proof, we generalize  existing analyses of factor models \citep{Bai-factor-model-03,Bai-Ng-forecast,fan2013large,SW2002_JASA} by relaxing the stringent conditions that require all singular values of the latent components of $\bX$ to grow at the same rate, proportional to the dimension $p$.

Given the success in estimating $z^\T\beta$ via $x^\T\wh\theta$, a rather difficult $p$-dimensional problem,  one would expect that consistent estimation of $\beta_0$ is much easier. Surprisingly, Proposition \ref{lem_GLS} in Section \ref{sec_intercept} shows that this is not the case. The natural plug-in estimate of $\beta_0$ based on $\wh\theta$ and standard non-parametric estimates of the conditional means and label probabilities, always takes the value $-1/2$, regardless of the true value of $\beta_0$. The same is true for an estimate based on empirical risk minimization. Simulations confirm that this problematic behavior  leads to an inferior classifier.   In Section \ref{sec_intercept} we offer a simple rectification and propose to estimate $\beta_0$ using  an independent hold-out sample.  
In Proposition \ref{prop_intercept} of Section \ref{sec_rates}, we derive the consistency of our proposed estimator of $\beta_0$. Finally, in Theorem \ref{thm:rates} we establish the rate of convergence of the excess-misclassification risk of our proposed classifier  and discuss its minimax-optimal properties in Remark \ref{rem:minimax}. 

In view of the optimal guarantees of the proposed classifier in over-parametrized setting, we also find  an interesting observation on its interpolation property. 
Specifically, Lemma \ref{lem_beta_0_sign} in Section \ref{sec_interpolation} and  our discussion in Section \ref{sec_label}   reveal that its interpolation property crucially depends on the way we encode the labels. For instance, interpolation always happens if we 
encode $Y$ as  $\{-1, 1\}$, whereas this is not always the case for the $\{0,1\}$ encoding scheme unless the majority class is encoded as $0$.  This suggests that interpolation  is a rather arbitrary property.\\  

The paper is organized as follows. In Section \ref{sec_background} we formally introduce the statistical model. We discuss the interpolation property of the proposed classifier and introduce the estimator of the intercept  in Section \ref{sec_interpolation_intercept}. Section \ref{sec_rates} is devoted to study the rate of convergence of the excess risk of the proposed classifier. Section \ref{sec_sim} contains simulation studies. The main proofs are deferred to Section \ref{sec_proofs} while auxiliary lemmas are stated in Section \ref{app_aux}.

\subsection{Notation} 

We use the common notation  $\varphi(x)=\exp(-x^2/2) / \sqrt{2\pi}$ for the standard normal density, and denote by $\Phi(x)=\int  \varphi(t)\1\{ t\le x\} \, {\rm d} t$  its c.d.f.. 

For any positive integer $d$, we write $[d] := \{1,\ldots, d\}$.
For two numbers $a$ and $b$, we write $a\wedge b = \min\{a, b\}$ and $a\vee b =\max\{a,b\}$. 
For any two sequences $a_n$ and $b_n$, we write $a_n\lesssim b_n$ if there exists some constant $C$ such that $a_n \le Cb_n$. 
The notation $a_n\asymp b_n$ stands for $a_n \lesssim b_n$ and $b_n \lesssim a_n$. We often write $a_n \ll b_n$ if $a_n = o(b_n)$ and $a_n\gg b_n$ for $b_n = o(a_n)$. 

For any vector $v$, we use $\|v\|_q$ to denote its $\ell_q$ norm for $0\le q\le \i$.  We also write $\|v\|_Q^2 = v^\T Q v$ for any commensurate, square matrix $Q$. For any real-valued matrix $M\in \RR^{r\times q}$, we use $M^+$ to denote the Moore-Penrose inverse of $M$, and $\sigma_1(M)\ge \sigma_2(M)\ge \cdots \ge \sigma_{\min(r,q)}(M)$ to denote the singular values of $M$ in non-increasing order. We define the operator norm $\|M\|_{\op}=\sigma_1(M)$. 
For a symmetric positive semi-definite matrix $Q\in \RR^{p\times p}$, we use $\lambda_1(Q)\ge \lambda_2(Q)\ge \cdots \ge \lambda_p(Q)$ to denote the eigenvalues of $Q$ in non-increasing order.  
We use $\bI_d$ to denote the $d\times d$ identity matrix and use $\b1_d$ ($\0_d$) to denote the vector with all ones (zeroes). For $d_1\ge d_2$, we use $\cO_{d_1\times d_2}$ to denote the set of all $d_1\times d_2$ matrices with orthonormal columns.

\section{Background}\label{sec_background}

Suppose our training data consists of independent copies of the pair $(X,Y)$ with features $X\in \RR^p$ according to 
\begin{equation}\label{model_X}
	X = AZ + W
\end{equation}
and labels $Y\in\{0,1\}$. 
Here $A$ is a  deterministic, unknown $p\times K$  loading matrix,
$Z\in \RR^K$ are unobserved, latent factors and $W$ is random noise.
We assume {\em throughout this study} the following set of assumptions:
\begin{itemize}
	\item [(i)] $W$ is independent of both $Z$ and $Y$
	\item[(ii)] $\EE[Z]=\0_K$, $\EE[W]=\0_p$  
	\item[(iii)] $A$ has rank $K$
	\item[(iv)] 
	$Z$ is a mixture of two Gaussians 
	\begin{equation}\label{model_YZ}
		Z \mid  Y = k \sim N_K(\a_k, \szy),\qquad \PP(Y=k) = \pi_k, \qquad k\in \{0,1\}
	\end{equation}
	with different means $\a_0:= \EE [Z| Y=0] $ and $\a_1:= \EE[Z| Y=1] $, but with the same, strictly positive definite covariance matrix 
	\begin{equation}\label{def_szy}
		\szy:= \Cov{(Z|Y=0)}= \Cov{(Z|Y=1)}
	\end{equation}	
	\item[(v)] $W = \sw^{1/2} V$, 
	for some semi-positive definite matrix $\sw$, with 
	$\EE[V]=\0_p$, $\EE[ VV^\T] = \bI_p$ and
	$\EE[\exp(u^\T V)] \le \exp(\sigma^2/2)$ for all $\|u\|_2=1$, for some $0<\sigma<\infty$.
	\item[(vi)]  $\pi_0=\PP\{Y=0\}$ and $\pi_1=\PP\{Y=1\}$ are fixed and strictly positive.
\end{itemize}

We emphasize that the distributions of $X$ given $Y$ are not necessarily Gaussian.
This  mathematical framework allows for a substantial dimension reduction in classification for $K\ll p$ as is evident from the inequality 
\begin{align}\label{eq_RxRz}
	R_x^*  := \inf_{g} \PP\{ g(X) \ne Y \} & ~\ge~  R_z^*:=\inf_{h} \PP\{ h(Z) \ne Y \} 
\end{align}  
in terms of the Bayes' misclassification errors, see
\citet[Lemma 1]{BW22}.
As in \cite{BW22}, we 
first change the classification problem into a regression problem by drawing a connection of the Bayes rule to a quantity that can be identified via regressing $Y$ onto $Z$.
We denote by  $\sz = \EE[ZZ^\T]$ the {\em unconditional} covariance matrix of $Z$ and we define
\begin{align}\label{eq_beta}
	\beta &= \pi_0\pi_1  \sz^{-1}(\a_1 - \a_0),\\\label{eq_beta_0}
	\beta_0 &= -{1\over 2}(\a_0+\a_1)^\T \beta + \left[
	1 - (\a_1-\a_0)^\T \beta
	\right]\pi_0\pi_1 \log{\pi_1 \over \pi_0}.
\end{align}

\begin{prop}\label{prop_ls_rule}
	Let 
	$\beta,\beta_0$ be defined in 
	(\ref{eq_beta}). 
	The (Bayes) rule
	\begin{equation}\label{Bayes_rule}
		g_z^*(z)= \1 \{
		z^\T \beta +\beta_0 > 0\}
	\end{equation}
	minimizes the misclassification error $\PP\{ g(Z)\ne Y\}$ over all  $g:\RR^K\to\{0,1\}$.
	Furthermore, 
	\[
	\beta = \sz^{-1}\EE[ZY].
	\]
\end{prop}
\begin{proof}
	See \citet[Proposition 4]{BW22}.
\end{proof}

Note that $g_z^*$ in \eqref{Bayes_rule} has different form from the canonical LDA rule based on $\szy^{-1}(\a_1-\a_0)$ with $\szy$ being the {\em conditional} covariance matrix.  The advantage of expressing $g_z^*$ in terms of $\beta$ lies in the fact  that $\beta$ can be obtained by simply regressing $Y$ on $Z$, hence  there is no need to estimate $\szy^{-1}$.
If $\bZ = (Z_1, \ldots, Z_n)^\T \in \RR^{n\times K}$  were observed,
it is natural to use 
the least squares estimator $\bZ^+ \bY= (\bZ^\T \bZ)^{+}\bZ^\T \bY$ to estimate $\beta$. 
However, we only have access to   the 
$n\times p$ data-matrix 
\[ \bX=(X_1,\ldots,X_n)^\T
\]
based on $n$ independent observations $X_i\in\RR^p$ from (\ref{model_X})
and the vector 
\[ \bY=(Y_1,\ldots,Y_n)^\T \in \{0,1\}^n\]
of  labels $Y_i$ corresponding to the rows $X_i$ of $\bX$.
For a new feature $x\in \RR^p$ generated from model (\ref{model_X}), the inner-product
$z^\T \beta$ is estimated by $x^\T \wh\theta$   based on the minimum-norm interpolator, also generally termed as the
generalized least squares (GLS) estimator,
\begin{equation}\label{def_GLS}
	\wh\theta = \bX^+ \bY = (\bX^\T \bX)^+ \bX^\T \bY.
\end{equation}
In \cite{BW22}, the author used  Principal Component Regression (PCR)  instead of the GLS-estimate $\wh \theta$ to estimate $z^\T \beta$. The intuition of PCR lies in approximating the span of $\bZ$ by that of the first $K$ principal components of $\bX$. 
Thus PCR is a more complicated method because of estimating the latent dimension $K$, but is often minimax optimal (sometimes using a slight, yet necessary modification involving data-splitting). 
The GLS-estimate $\wh \theta$ on the other hand does not require selection of $K$, and is free of tuning parameters. Because of this, it is far from clear whether the span of $\bX$ approximates that of $\bZ$.  It is therefore of great interest to see whether the method based on $\wh \theta$ works, and if so, whether it is minimax optimal.\\


The classifier that we study in this paper has the form in \eqref{def_g_hat_intro}. We refer to this classifier as the GLS-based classifier. 
Estimation of the intercept $\beta_0$ is discussed in Section \ref{sec_intercept}.
Our goal is to analyze its misclassification error 
relative to the oracle risk 
$R_z^*$ in (\ref{eq_RxRz}).

\begin{remark}
	For the oracle risk, 
	we have the explicit expression
	\begin{equation}\label{eqn_Bayes_risk}
		R_z^* =  
		1 - \pi_1\Phi \left({\Dt\over 2} + {\log {\pi_1 \over \pi_0} \over \Dt}\right) - \pi_0 \Phi \left({\Dt\over 2} - {\log {\pi_1 \over \pi_0} \over \Dt}\right),
	\end{equation}
	see, for instance, \citet[Section 8.3, pp 241--244]{Izenman-book}, based on 
	the Mahalanobis distance
	\begin{equation}\label{def_Dt}
		\Dt^2 := (\a_1-\a_0)^\T \C^{-1} (\a_1 - \a_0) = \| \a_1-\a_0\|_{\C^{-1}}^2
	\end{equation} 
	In particular, when $\pi_0 = \pi_1$, the expression in (\ref{eqn_Bayes_risk})  simplifies to
	$
	R_z^* = 1 - \Phi\left(\Dt/2 \right).
	$   
\end{remark}

\section{Interpolation and estimation of the intercept}\label{sec_interpolation_intercept}

We first review the interpolation property of the minimum-norm interpolator $\wh\theta$ and discuss the interpolation of the classifier $\1\{x^\T \wh\theta + \beta_0 > 0\}$ based on the true intercept. We then show that a natural plug-in estimator of $\beta_0$ is surprisingly inconsistent,  which leads us to propose a different  estimator of $\beta_0$. Finally, we show that the interpolation property of the classifier is connected with the way we encode our labels. This reveals that the interpolation property is a rather arbitrary artifact.


\subsection{Interpolation}\label{sec_interpolation}

The theoretical  performance of the GLS estimator (\ref{def_GLS}) including its interpolation property is now well understood  in linear regression settings $\bY = \bX\theta + \beps$ when the feature dimension $p$ is much larger than the sample size $n$ (see, for instance, \cite{belkin2018overfitting,bartlett2019,hastie2019surprises,BSW22} and the references therein). 
%
The following result shows that with high probability, $\rank(\bX) = n$ 
whence $\wh\theta$ interpolates the training data, 
provided that
$K\le n$, $\tr(\sw)>0$ and ${\rm r_e}(\sw) := \tr({\sw}) / \| \sw\|_{\op}\gg n$. 
\begin{prop}\label{lem_lb_sigma_n_X} 
	Assume $n \ge K$. Then, there exist finite, positive constants $C,c$ depending on $\sigma$ only, such that,
	provided ${\rm r_e}(\sw) \ge C n$,
	\[   
	\PP \left\{ \sigma_n^2({\bX}) 
	\ge \frac{1}{8} \text{\rm tr}(\Sigma_W)\right\}\ge 1 - 3\exp(-c~ n ) \]
	and thus,   if  in addition $\tr(\sw)>0$, $\wh\theta$  interpolates: $ \PP \{ \bX \wh\theta = \bY\}  \ge 1-3\exp(-c ~ n)$.
\end{prop}
\begin{proof}
	See \citet[Proposition 14]{BSW22}. 
\end{proof}
From Proposition \ref{lem_lb_sigma_n_X}, $\wh\theta$ interpolates the training data provided that $K\le n$ and  ${\rm r_e}(\sw) \gg n$. The latter is connected to the over-parametrization. To see this, suppose that $\sw$ has bounded eigenvalues, that is, for some absolute constants $0<c\le C<\i$,
\begin{equation}\label{cond_sw}
	c \le \lambda_p(\sw) \le \lambda_1(\sw) \le C.
\end{equation}
We see that ${\rm r_e}(\sw) \asymp p$, whence ${\rm r_e}(\sw) \gg n$ reduces to the over-parametrized setting $p\gg n$. \\

Given the interpolation property of $\wh\theta$, we immediately see that $X_i^\T \wh \theta + \bar \beta_0 = Y_i + \bar\beta_0>0$ if and only if $Y_i=1$, for all $i\in [n]$, as long as  $\bar\beta_0\in (-1,0]$. Hence, for {\em any} $\bar\beta_0\in (-1,0]$ (including zero intercept advocated in the recent literature), the classifier $\1\{x^\T \wh\theta + \bar \beta_0>0\}$ would perfectly classify  the training data.  
A natural question is to see whether or not the classifier $\1\{x^\T \wh\theta + \beta_0>0\}$ that uses the true intercept $\beta_0$ would yield zero training error. This is equivalent with verifying if $\beta_0 \in (-1,0]$. The following lemma provides the answer, which surprisingly depends on the way we encode.
\begin{lemma}\label{lem_beta_0_sign}
	The intercept $\beta_0$ in (\ref{eq_beta_0}) satisfies 
	\begin{equation*}
		\sgn(\beta_0) = \sgn\left({1\over 2} - \pi_0\right), \qquad 
		|\beta_0| \le  \left|
		{1\over 2} - \pi_0
		\right|.
	\end{equation*}
\end{lemma} 
\begin{proof}
	By the identity (\ref{eq_beta_eta}), the definition  (\ref{def_Dt}) of $\Dt$, the identity 
	\begin{equation*}\label{eq_beta_eta}
		\sz^{-1}(\a_1 - \a_0) = {1\over 1 + \pi_0\pi_1\Dt^2}\szy^{-1}(\a_1 - \a_0), 
	\end{equation*}
	(see, the proof of Lemma 14 in \cite{BW22}) and the fact that $\EE[Z] = \pi_0\a_0 + \pi_1 \a_1 = \0_K$, we have, 
	after a bit of simple algebra, 
	\begin{equation}\label{eq_beta_0_general}
		\beta_0 =  \left[
		{1\over 2} - \pi_0 + {1\over \Dt^2}\log{\pi_1 \over \pi_0}
		\right]{ \pi_0\pi_1 \Dt^2 \over 1 +  \pi_0\pi_1 \Dt^2}.
	\end{equation}
	It is readily seen that $\sgn(\beta_0) = \sgn\left({1/ 2} - \pi_0\right)$. 
	
	For the second claim, suppose $0<\pi_0 < 1/2$. After rearranging terms, we find
	\[
	\beta_0 =  \left[{1\over 2} - \pi_0 - {1\over 1 + \pi_0\pi_1 \Dt^2}\left(
	{1\over 2} - \pi_0 - \pi_0\pi_1\log{\pi_1 \over \pi_0}
	\right)\right].
	\]
	Since the term in parenthesis is positive,  $1/2 - \pi_0 - \pi_0\pi_1\log(\pi_1 /\pi_0) > 0$, we conclude that $0<\beta_0 <  (1/2-\pi_0)$. A similar argument can be used to prove 
	$ (1/2-\pi_0)<\beta_0 < 0$ for
	the case $1/2 \le \pi_0 < 1$.    
\end{proof}

Lemma \ref{lem_beta_0_sign} has  the curious consequence that {\em only} if we encode the majority class as $0$, does the rule $\1\{ x^\T \wh \theta + \beta_0 > 0\}$ have zero training error under interpolation $\bX \wh \theta= \by$. In Section \ref{sec_label} we will show that the interpolation property also depends on the values we use to encode the labels, and is therefore rather arbitrary. 

\subsection{Estimation of the intercept}\label{sec_intercept}

Given $\wh \theta=\bX^+ \bY$, assuming $x^\T \wh\theta$ is close to $z^\T \beta$, we can naively estimate $\beta_0$ in (\ref{eq_beta_0}) by the following plug-in estimator,
\begin{equation}\label{def_theta_0_hat}
	\wh\beta_0  := -{1\over 2}(\wh \mu_0+\wh \mu_1)^\T \wh\theta+   \left[
	1   -(\wh \mu_1-\wh \mu_0 )^\T\wh\theta \ 
	\right] \wh \pi_0\wh \pi_1 \log{\wh \pi_1 \over \wh \pi_0}
\end{equation}
based on standard non-parametric estimates
\begin{align}\label{def_pi_hat}
	\wh n_k = \sum_{i=1}^n \1{\{Y_i = k\}},\quad \wh\pi_k = {\wh n_k \over n},\quad 
	\wh \mu_k = {1\over \wh n_k}\sum_{i=1}^n X_i\1\{ Y_i = k\}, 
	\quad k\in \{0,1\}. 
\end{align}
This leads to the naive classifier
\begin{equation}\label{eq_g_hat}
	\wh g(x) :=  \1\{ x^\T \wh \theta + \wh \beta_0 > 0\}.
\end{equation}
The following lemma shows that $\wh \beta_0=-1/2$, irrespective of the true value of $\beta_0$, whenever $\wh\theta$ interpolates. On the one hand, this means that the naive classifier $\wh g(x)$ always interpolates as $\wh \beta_0\in(-1, 0]$. On the other hand, it shows that $\wh\beta_0$ clearly is an inconsistent estimate of $\beta_0$ in general.

\begin{prop}\label{lem_GLS}
	Let 
	$\wh\beta_0$ be defined in (\ref{def_theta_0_hat}). On the event $\{ \bX\wh\theta = \bY\} $ where $\wh \theta$ interpolates,   we have $\wh\beta_0=-1/2$.
\end{prop}
\begin{proof}
	We define the vectors $v,w\in \RR^n$ as    \begin{align*}
		v &= {1\over \wh n_1} \left(\1\{Y_1=1\},\ldots,\1\{Y_n=1\}\right)^\T
		+ {1\over \wh n_0}\left(\1\{Y_1=0\},\ldots,\1\{Y_n=0\}\right)^\T,\\
		w &= {1\over \wh n_1} \left(\1\{Y_1=1\},\ldots,\1\{Y_n=1\}\right)^\T
		- {1\over \wh n_0}\left(\1\{Y_1=0\},\ldots,\1\{Y_n=0\}\right)^\T.
	\end{align*}   
	%
	By (\ref{def_pi_hat}), we can write
	\[
	\wh\mu_1 + \wh\mu_0 = \bX^\T v,\qquad \wh\mu_1 - \wh\mu_0 = \bX^\T w 
	\]
	and hence
	\begin{align*}
		\wh\beta_0 &= -{1\over 2}(\wh \mu_0+\wh \mu_1)^\T \wh\theta+ \left[1 - (\wh \mu_1-\wh \mu_0 )^\T\wh\theta \  \right]\wh \pi_0\wh \pi_1\log{\wh \pi_1 \over \wh \pi_0}\\
		&=  -{1\over 2}v^\T \bX\wh\theta+ \left[1 - w^\T\bX\wh\theta \  \right] \wh \pi_0\wh \pi_1\log{\wh \pi_1 \over \wh \pi_0}
	\end{align*}
	Use $\bX\wh\theta = \bY$ and $v^\T \bY = w^\T\bY = 1$ to obtain
	\begin{align*}
		\wh\beta_0 &=- {v^\T \bY \over 2} + \left[
		1 - w^\T \bY 
		\right] \wh \pi_0\wh \pi_1 \log{\wh \pi_1 \over \wh \pi_0} =- {1 \over 2}.
		%
	\end{align*}
	This proves the our claim.
\end{proof}

Proposition \ref{lem_GLS} implies that 
the naive classifier $\wh g(x)$ from (\ref{eq_g_hat}) cannot be consistent in general due to the inconsistency of $\wh\beta_0$,  for the same reason that taking no intercept ($\wh \beta_0=0$) is inconsistent. 
From the proof of Proposition \ref{lem_GLS}, we see that this phenomenon still exists in the classical LDA setting, where $X \mid Y$ is Gaussian, should any interpolating regression estimate such as $\wh\theta$ be employed and plugged in (\ref{def_theta_0_hat}). 
The inconsistency of GLS-based LDA is in sharp contrast to its magical performance in  factor regression models, see \cite{BSW22,  bing2020prediction}. 
This phenomenon is corroborated in our simulation study in Section \ref{sec_sim}.\\

While the estimator $\wh\beta_0$ in (\ref{def_theta_0_hat}) based on the GLS estimator $\wh\theta$ is clearly inconsistent for estimating the intercept $\beta_0$, we will show in Theorem \ref{thm_rate_Atheta} of the next section that  $\wh \theta$ does estimate the direction $\beta$ consistently. In other words, failure of consistently estimating the intercept $\beta_0$ is the only cause for the subpar misclassification rate of $\wh g(x)$. 
In the symmetric case $\pi_0=\pi_1=1/2$, we have $\beta_0=0$ and we prove in Corollary \ref{cor_excess_ris} of the next section that the classifier $ \1\{ x^\T \wh\theta +0 > 0\}$
is consistent, often even minimax optimal.

In general, when $\pi_0\ne \pi_1$,  we should 
choose a different estimate for $\beta_0$.
Our solution is to use an independent hold-out sample $(X_i', Y_i')$, $i\in [n']$, for some integer $n'>0$, 
to estimate $\beta_0$ by
\begin{equation}\label{def_theta_0_hat2}
	\widetilde\beta_0  := -{1\over 2}(\widetilde \mu_0+\widetilde \mu_1)^\T \wh\theta+   \left[
	1   -(\widetilde \mu_1-\widetilde \mu_0 )^\T\wh\theta \ 
	\right] \wh \pi_0\wh \pi_1 \log{\wh \pi_1 \over \wh \pi_0}
\end{equation}
with $\wh\theta$ from (\ref{def_GLS}), $\wh \pi_k$ from (\ref{def_pi_hat}) and 
\begin{align}\label{def_beta0tilde}
	\widetilde \mu_k = {1\over \widetilde n_k}\sum_{i=1}^{n'} X_i'\1\{ Y_i' = k\}, 
	\quad \widetilde n_k = \sum_{i=1}^{n'} \1{\{Y_i' = k\}},
	\quad k\in \{0,1\}. 
\end{align}
This simple modification ensures that $\wt \beta_0$ is a consistent estimator of $\beta$, as shown in Proposition \ref{prop_intercept} of the next section. Furthermore, 
the corresponding classification rule 
\begin{equation*} 
	\1\{ x^\T \wh\theta  + \wt \beta_0  > 0\}
\end{equation*}
is consistent, and even minimax-optimal, in many scenarios
(see Remark \ref{rem:minimax}), although it no longer necessarily classifies the training data perfectly.

\begin{remark}[Alternative estimation of $\beta_0$]\label{rem_erm}
	It is essential for establishing consistency of $\wt \beta_0$ that the estimates $\wt\mu_0, \wt\mu_1$ and $\wh \theta$ are statistically independent. 
	Alternatively, we could use the hold-out sample to estimate the intercept via minimizing the empirical risk
	\begin{align}\label{erm}
		\sum_{i=1}^{n'} \1 \{ (2Y_i^\prime-1) (\wh\theta~^\T X_i^\prime + \beta_0) <0 \} 
	\end{align}
	over $\beta_0\in \RR$. It is clear that we need to use the hold-out sample in (\ref{erm}) as well, since minimizing over the same training data would lead to interpolation and any value within $(-1,0]$ whenever $\wh\theta$ interpolates. In our simulation of Section \ref{sec_sim}, we found that not only can we compute $\wt\beta_0$ in (\ref{def_theta_0_hat2})   much faster comparing to (\ref{erm}), it also leads to better classification performance as well. 
\end{remark}

\subsection{Effect of label encoding on interpolation} \label{sec_label}

In this section we discuss how the (in sample) interpolation property depends on the way we encode our labels. Let $b>a$ be any scalars and consider the encoding $Y\in \{a,b\}$.    
In the following lemma we show that the optimal decision boundary in the latent space is independent of the particular encoding. This reassures us that the optimal classification rule does not depend on the encoding.

\begin{lemma}\label{lemma:plane}
	Given the encoding $Y\in \{a,b\}$ with any $a<b$, the Bayes rule is 
	$$
	(b-a)\1\left\{(b-a)\left(z^\T \beta + \beta_0 \right) > 0\right\} + a
	$$ 
	with $\beta$ and $\beta_0$ defined in \eqref{eq_beta} and \eqref{eq_beta_0}, respectively. In particular, the optimal boundary $\{z\mid (b-a)(z^\T \beta + \beta_0) = 0\}$ is invariant to any pair $(a,b)$ with $b>a$.
\end{lemma}
\begin{proof}
	From the proof of Proposition 4 in \cite{BW22}, we can deduce that the optimal hyperplane is $\{z\mid z^\T \beta^{(a,b)} + \beta^{(a,b)}_0 = 0\}$ where 
	\begin{align}\label{eq_beta_general} 
		\beta^{(a,b)} &= (b-a) \pi_0\pi_1 \sz^{-1}(\a_1 - \a_0),\\\nonumber
		\beta^{(a,b)}_0 &=  -{1\over 2}(\a_0+\a_1)^\T \beta + \left[
		b-a - (\a_1-\a_0)^\T \beta
		\right]\pi_0\pi_1 \log{\pi_1 \over \pi_0}.
	\end{align}
	This proves our claim.
\end{proof}

However, it is a different story for the interpolation property with 
\begin{equation}\label{rule_label}
	(b-a) \1\{x^\T \wh\theta + (b-a)\beta_0 > 0\} + a.
\end{equation} 
Note this is the classifier corresponding to 
$\1\{x^\T \wh\theta + \beta_0 > 0\}$ under the encoding $Y\in\{0,1\}$.
Given that $\bX\wh\theta = \bY$, the classifier in \eqref{rule_label} has zero training error if and only if 
\begin{equation}\label{cond_label}
	(b-a)\beta_0 \in (-b, -a].
\end{equation}
Indeed, for any $i\in [n]$, observe that $X_i^\T\wh \theta + (b-a)\beta_0 >0$ if and only if $Y_i+(b-a)\beta_0>0$. Whether the classifier in \eqref{rule_label} interpolates is thus equivalent to whether \eqref{cond_label} holds. Below we use Lemma \ref{lem_beta_0_sign} to compare our $\{0,1\}$ encoding with another popular encoding $\{-1,1\}$.
\begin{itemize}
	\item  As noted earlier, if we use 
	the $\{0,1\}$ encoding, \eqref{cond_label} becomes $\beta_0\in (-1,0]$ and the classifier in \eqref{rule_label} has zero training classification error if and only if we encode the majority class as $0$. 
	
	\item 
	If we use the encoding 
	$\{-1,1\}$,
	then \eqref{cond_label} becomes $\beta_0\in (-1/2,1/2]$ which holds in view of Lemma \ref{lem_beta_0_sign}. Therefore,
	the classifier in \eqref{rule_label} {\em always} has zero training classification error. Furthermore,   we expect that all rules
	based on \eqref{rule_label} with $\beta_0$ replaced by a consistent estimate
	to classify the training data perfectly.
\end{itemize}

\section{Rates of convergence for the excess risk}\label{sec_rates}
In this section, we analyze the excess risk of the classifier
\begin{equation}\label{def_g_td}
	\wt g(x)=\1\{ x^\T \wh\theta  + \wt\beta_0>0 \}
\end{equation} 
for $\wh\theta$ defined in (\ref{def_GLS}) and $\wt \beta_0$ defined in (\ref{def_theta_0_hat2}). 
We define the excess risk of this classifier as $\PP\{ \wt g(X)\ne Y\}- R_z^*$ with $R_z^*$ given in \eqref{eq_RxRz} (see, also, \eqref{eqn_Bayes_risk}).  Following \citet[Theorem 5]{BW22}, we have, for all $t>0$,  
\begin{align}\label{disp_margin}
	\PP\{ \wt g(X)\ne Y\} -  R_z^*
	&~ \le  ~ \PP \left\{| X^\T \wh\theta - Z^\T \beta   + \wt \beta_0 - \beta_0 | > t \right\} +      P(\Dt, t) 
\end{align}
where, with $c_\Dt := \Dt^2 + (  \pi_0\pi_1)^{-1}$,
\begin{align}\nonumber
	P(\Dt, t) &=  
	t ~ c_\Dt  \Bigl[ \pi_0   \PP\{ - t< Z^\T \beta + \beta_0 <0 \mid Y=0\} 
	+     \pi_1     \PP\{ 0< Z^\T \beta + \beta_0 < t\mid Y=1\} \Bigr].
\end{align} 
We see that the excess risk depends on:\\ (a) the probabilistic behavior of the boundary of the optimal hyperplane, and
\\ (b) the quality of our estimate of the optimal hyperplane in $\RR^K$. 

Part (a) is expressed in the quantity $P(\Dt, t)$  and reflects the intrinsic difficulty of the classification problem.
As in \cite{BW22}, we can distinguish four cases:  Let  $c$ and $c'$ be some absolute positive constants.  For any $t>0$,
\begin{align}\label{def_P_Dt_t}
	P(\Dt, t)
	&=
	\begin{cases} 
		t^2  & \text{ if $\Dt\asymp 1$;}\\
		t^2 \exp\left\{-[c +o(1)]\Dt^2\right\} & \text{ if $\Dt\to\i$ and $t \to 0$;}\\
		t^2 \exp\left\{-[c' + o(1)] / \Dt^2\right\} & \text{ if $\Dt\to0$, $\pi_0\ne\pi_1$  and $t \to 0$;}\\
		t  \min\{1,t / \Dt\} & \text{ if $\Dt\to0$ and $\pi_0=\pi_1$}.
\end{cases}\end{align}
The case $\Dt\to\i$ can be considered an easy case as the Bayes error $R_z^*$ vanishes exponentially fast in $\Dt^2$. In this case, $P(\Dt, t)$ also vanishes exponentially fast.  We also note that the proof of Lemma \ref{lem_beta_0_sign} reveals that $|\beta_0-(1/2 - \pi_0)|\to0$ in this case; that is, the true intercept is easier to estimate as well.  
The case $\Dt\to0$ and $\pi_0\ne \pi_1$ has a trivial Bayes risk $R_z^*\to \min\{\pi_0,\pi_1\}$, hence is easy to classify since the optimal Bayes rule classifies only according to the largest unconditional class probability $\pi_k$, irrespective of the covariate $x$. In this case the intrinsic difficulty $P(\Dt, t)$ goes to zero exponentially fast in $1/\Dt^2$. The case $\Dt\to0$ and $\pi_0=\pi_1$ is impossible to classify as $R_z^* \to 1/2$, corresponding to random guessing. For this reason, we concentrate on the intermediate case $\Dt\asymp1$ in this work.

For part (b),  we need to control 
\begin{align}\label{identiteit}
	|X^\T \wh\theta - Z^\T \beta |
	\le |Z (A^\T \wh \theta-\beta)| + |W^\T \wh \theta|,
\end{align}
the error of predicting the `direction' $Z^\T \beta$, as well as  $|\wt \beta_0 - \beta_0|$, the error of estimating the intercept. 
%
%
The following proposition provides upper bounds of the two terms on the right of (\ref{identiteit}).

\begin{prop}\label{prop:GsG}
	For every $\delta>0$, we have
	\begin{align*}
		\PP\left[   \left| W^\T \wh\theta \right| \ge \sigma \sqrt{2    \| \wh \theta\|_{\sw}\log( 1/\delta ) }   \, \right] \le  2\delta
	\end{align*}
	and
	\begin{align*}
		\PP\left[  \left | Z^\T (A^\T \wh\theta -\beta) \right| \ge 
		\left( \frac{1}{2\pi_0\pi_1} + \sqrt{ \log(1/\delta) } \right) \| A^\T \wh \theta -\beta \|_{\sz }
		\, \right] \le 4\delta .
	\end{align*}
\end{prop}
\begin{proof}
	The bounds follow from the independence of $\wh \theta,Z$ and $W$, the Gaussian assumption (iv) of $Z\mid Y=k$ and subGaussian distribution (v) of $W$.
	See, for instance, the proof of Proposition 6 in \cite{BW22}.
\end{proof}

For the later analysis, we will apply Proposition \ref{prop:GsG} with $\delta= 1/n^c$ for some absolute constant $c>0$, which will result in a multiplicative $\log(n)$ term for the corresponding rates. From Proposition \ref{prop:GsG}, it is clear that we need to bound $\| \wh \theta\|_{\sw}$ and $\| A^\T \wh \theta -\beta\|_{\sz}$.

\begin{prop}\label{prop_hattheta}
	Assume $n\ge K$. Then, there exist finite, positive constants $C,c$ depending on $\sigma$ only, such that, provided ${\rm r_e}(\sw) \ge C n$, 
	$$
	\PP\left\{  		\| \wh\theta\|_{\sw} ^2  ~ \le ~ {8n\over {\rm r_e}(\sw)} \right\}\ge 1-3\exp(-c~ n).
	$$ 
\end{prop}

\begin{proof} 
	By definition, 
	\[
	\|\sw^{1/2}\wh\theta\|_2^2 = 	\|\sw^{1/2}\bX^+\bY \|_2^2 \le \|\sw\|_\op {\|\bY\|_2^2 \over \sigma_n^2(\bX)}.
	\]
	The result follows from the inequality $\|\bY\|_2 \le \sqrt n$ and  Proposition \ref{lem_lb_sigma_n_X}. 
\end{proof}

From Proposition \ref{prop_hattheta}, $\|\wh\theta\|_{\sw} \to 0$,  with overwhelming probability, is ensured if  $n/ {\rm r_e(\sw)}\to 0$ as $n\to \i$. {The latter holds in the over-parametrized setting $p\gg n$ with condition \eqref{cond_sw}.}\\

Bounding  $\|A^\T \wh\theta - \beta\|_{\sz}$ from above on the other hand is the main difficulty in our analysis. To present our result, we write the non-zero eigenvalues of $A\sz A^\T$ as $\lambda_1\ge \cdots \ge \lambda_K$ and 
its  condition number $\lambda_1 / \lambda_K$ as $ \kappa$. We further define  
\begin{equation}\label{def_xi}
	\xi:= {\lambda_K \over \|\sw\|_\op}
\end{equation}
as  the signal-to-noise ratio of predicting the signal $Z$ from $X = AZ + W$ in the presence of the noise $W$. At the sample level, the noise level $\|\sw\|_\op$ gets inflated  in the sense that 
\begin{equation}\label{bd_W_op}
	\PP\left\{
	{1\over n}\| \bW^\T \bW \|_{\rm op} \le 12 \sigma^2 \|\sw\|_\op \left( 1  + {{\rm r_e}(\sw) \over n}\right)
	\right\} = 1 - \exp(-n),
\end{equation}
see, Lemma \ref{lem_op_norm} in the Appendix. Finally, we set
\begin{align}\label{def_omega_n}
	\omega_n &=  \sqrt{K \log(n) \over n} + 
	\sqrt{K \log(n) \over n} { \kappa ~ n\over {\rm r_e}(\sw)} +  \sqrt{n\over {\rm r_e}(\sw)} 
	+ {{\rm r_e}(\sw) \over n}  { \sqrt{\kappa}  \over \xi}+{\kappa\over\sqrt{\xi}}.
\end{align}

\begin{thm}\label{thm_rate_Atheta}
	Assume the following holds as $n\to \i$,
	\begin{align}\label{assumps}
		{K \log(n) \over  n} \to 0,\quad  {n \over {\rm r_e}(\sw)} \to 0\quad \text{ and }\quad {\kappa  \over \xi}\left( 1  + {{\rm r_e}(\sw) \over n}\right) \to 0. 
	\end{align}
	Then,   for any constant $c>0$, there exists a constant $C=C(\sigma)<\infty$ such that 
	\begin{align*}	 		
		\PP\left\{ 	\| A^\T \wh\theta - \beta\|_{\sz}   \le C  ~ \omega_n \right\} =1 -\cO( n^{-c})
	\end{align*}
\end{thm} 
We refer to Section \ref{sec:proof:main} for the proof and  explain 
the main difficulties of the analysis in Section \ref{sec_technical} below. 
When $\kappa = \cO(1)$, the set of assumptions (\ref{assumps}) is needed to ensure $\omega_n\to0$ in (\ref{def_omega_n}). The first condition puts a restriction on the latent dimension relative to the sample size, the second condition  holds in the over-parametrized setting $p\gg n$ with \eqref{cond_sw}, while the third one requires the signal $\lambda_K$ of predicting $\bZ$ from $\bX$ to exceed the sample level noise (cf. (\ref{bd_W_op})).  
In Remark \ref{rem:minimax} we simplify the expression of $\omega_n$ and provide an interpretation for each term.

\begin{cor}\label{cor_pred_dir}
	Under condition (\ref{assumps}), for any constant $c>0$,
	there exists a constant $C=C(\sigma) < \infty$  such that
	\[
	\PP\left\{
	| X^\T \wh\theta  - Z^\T \beta | \ge  C   ~ \omega_n\sqrt{\log (n)}
	\right\} = \cO(n^{-c}).
	\]
\end{cor}
\begin{proof}
	Combination of  Proposition \ref{prop:GsG}, Proposition \ref{prop_hattheta} and Theorem \ref{thm_rate_Atheta} immediately yields the result.
\end{proof}

In case  $\pi_0=\pi_1=1/2$, the intercept $\beta_0 $ does not need to be estimated as $\beta_0=0$. 
Coupled with (\ref{disp_margin}) and (\ref{def_P_Dt_t}), Corollary \ref{cor_pred_dir}   immediately gives a bound on the excess risk of the classifier $\wh g(x) 
=\1\{ x^\T \wh\theta >0\} $ that uses a hyperplane through the origin.

\begin{cor}\label{cor_excess_ris}
	Under condition (\ref{assumps}), assume that $\Dt \asymp 1$ and $\pi_0=\pi_1=1/2$. The classifier $\wh g(x) =\1\{x^\T \wh \theta >0\}$ satisfies 
	\[
	\PP\{\wh g(X)\ne Y\} - R_z^* ~  \lesssim  ~   \omega_n^2\log (n).
	\]
\end{cor}
\medskip

Having successfully bounded  $|X^\T \wh \theta -Z^\T \beta| $, it remains  to bound  $|\wt \beta_0 - \beta_0|$ in order to apply    the excess risk bound (\ref{disp_margin}). Recall that $\wt \beta_0$ is given by \eqref{def_beta0tilde}.

\begin{prop} \label{prop_intercept}
Assume {$n'\gtrsim n$ and }$n/{\rm r_e}(\sw)\to0$ as $n\to\infty$. Then, for any $c>0$, there exists a $C=C(\sigma)<\infty$ such that
\begin{align*}
	\PP\left\{
	\left| \wt\beta_0-\beta_0 \right| 
	\le 
	C  \sqrt{   \frac{\log(n)}{{n}} }+ C\| A^\T \wh\theta - \beta\|_{\sz}  \right\}=1-\cO(n^{-c}).
\end{align*}
\end{prop}
\begin{proof} See Appendix \ref{ss_prop_intercept}.  
\end{proof}

For ease of presentation, we state our results for $n'\gtrsim n$.
Tracking our proof reveals that 
for any $n'\to \i$, the statement above continues to hold when we replace  $n$ by $n'\wedge n$.

Finally, we can state  our main result:

\begin{thm}\label{thm:rates}
Assume (\ref{assumps}),  $n' \gtrsim n$  and $\Dt\asymp 1$. Then  $\wt g(x)$ in \eqref{def_g_td} satisfies 
\[
\PP\left\{ \wt g(X)\ne Y\right\} - R_z^*  ~ \lesssim  ~   \omega_n^2\log (n).
\]
\end{thm}
\begin{proof}
The result follows from the excess risk bound (\ref{disp_margin}) and (\ref{def_P_Dt_t}) with $t=\omega_n \sqrt{\log(n)}$, Corollary \ref{cor_pred_dir} and Proposition \ref{prop_intercept}.
\end{proof}

\begin{remark}[Simplified excess risk bound and minimax optimality]\label{rem:minimax}
We now discuss the case when $p\gg n  \gg K$, $\Dt\asymp1$, $n\gtrsim n'$, ${\rm r_e}(\sw) \asymp p$ and $\kappa\asymp 1$. In this scenario, 
we will argue that our classifier $\wt g(x)$ in (\ref{def_g_td}) is minimax-optimal, provided both the ambient dimension $p$ and the signal-to-noise ratio $\xi={\lambda_K / \|\sw\|_\op}$ are large. We first observe that
we have  $n/{\rm r_e}(\sw) \asymp n / p \to 0$ and $\omega_n$ in (\ref{def_omega_n}) can be  simplified to 
\begin{equation}\label{omega_n_simp}
\omega_n^2  ~  \asymp ~  {K \log(n)  \over n} + 
{n\over p} 
+   \left({p\over n ~\xi}\right)^2 +{1 \over \xi}.
\end{equation}
We can follow the discussion after Theorem 3 in \cite{BW22}
to summarize the first, third and fourth terms.
The first term $K\log(n)/ n$ is the optimal rate of the excess risk when the latent factors  $\bZ$ and $Z$ were observable, hence it reflects the benefit of having a  hidden, low-dimensional structure with $K\ll n$.
The last term is essentially $R_x^* - R_z^*$ given by \eqref{eq_RxRz}, representing the irreducible error of predicting $Z$ from $X$ at the population level (see \eqref{def_xi}), while the third term can be interpreted as the error of predicting $\bZ$ from $\bX$ at the sample level (see \eqref{bd_W_op}). Different from   \cite{BW22}, is the second term $n/p$ in \eqref{omega_n_simp}, which  is due to using $X^\T \wh\theta$ to predict $Z^\T \beta$ instead of the PCR method advocated  in \cite{BW22}. It reveals the benefit of over-parametrization.


In view of \eqref{omega_n_simp},  
consistency of the classifier $\wt g$  in (\ref{def_g_td}) requires  the signal-to-noise ratio $\xi$ to be sufficiently large in the precise sense that  $\xi  \gg  p/n$. 
Furthermore, under the stronger assumption of  $\xi \gtrsim   (p/ n) \cdot ({n/K })^{1/2}$, we find 
\begin{align}\label{omega_n_simp_2}
\PP\{ \wt g(X)\ne Y\} - R_z^*  &\lesssim  ~ \omega_n^2 \log(n)  ~ \asymp ~ 
{K \over n}\log^2 (n) + {n \over p} \log(n).
\end{align}
We emphasize that this condition on $\xi$ is  a much weaker requirement than the condition $\xi \gtrsim p$  commonly made in the existing literature of high-dimensional factor models. See, for instance, \cite{Bai-factor-model-03,fan2013large,SW2002_JASA}). 
Finally, \citet[Theorem 1]{BW22} proves that the minimax optimal rate is proportional to $(K/n) + (1/\xi)$, and we find that   
the above  rate (\ref{omega_n_simp_2}) coincides with the minimax optimal rate, up to the logarithmic factor in $n$, in the high-dimensional setting  $p\gtrsim n^2 / K$.\\ 

\end{remark}

\begin{remark}[Rates of the excess risk for $\Dt \to \i$ and $\Dt \to 0$]\label{rem:phasetrans}
Since $\Dt\asymp 1$ is the most realistic and interesting case as discussed after display (\ref{def_P_Dt_t}), we state our main result in Theorem \ref{thm:rates} in this regime. Nevertheless, the results for $\Dt \to\i$ and $\Dt\to 0$ can be easily obtained by combining the excess risk bound (\ref{disp_margin}), display (\ref{def_P_Dt_t}), Proposition \ref{prop:GsG}, Theorem \ref{thm_rate_Atheta} and
Proposition \ref{prop_intercept},  in conjunction with (\ref{def_P_Dt_t}) for $t = \omega_n\sqrt{\log n}$. 
\end{remark}

\subsection{Technical difficulties in the proof of Theorem \ref{thm_rate_Atheta}}\label{sec_technical} 

On the event $\cE(\bZ,\bX)=\{\sigma_K^2(\bZ \sz^{-1/2})\ge n/2$, $\sigma_n^{2}(\bX) \ge \tr(\sw)/8\}$, which can be shown to hold with overwhelming probability (see, Proposition \ref{lem_lb_sigma_n_X} in Section \ref{sec_interpolation} and Lemma \ref{lem_Z} in Section \ref{app_aux}), the identities $\bZ^+\bZ = \bI_K$ and $\bX\bX^+ = \bI_n$ lead to the following chain of identities
\begin{align*}
A^\T \wh\theta  &= \bZ^+ \bZ A^\T \bX^+ \bY\\
&= \bZ^+ (\bX - \bW)\bX^+ \bY \\
&= \bZ^+\bY - \bZ^+\bW\bX^+\bY
\end{align*}
so that 
\begin{align*}
\|A^\T \wh\theta - \beta\|_{\sz} &\le \|\bZ^+\bY - \beta\|_{\sz} + \|\bZ^+\bW\bX^+\bY\|_{\sz}.
\end{align*}
The first term is relatively easy to analyze and it can be bounded as $\cO_\PP(\sqrt{K\log(n)/n})$. The second term is technically challenging to analyze because
commonly used arguments  render meaningless bounds. To appreciate the difficulty of the problem, let us consider three types of arguments in the simplified case $\sw = \bI_p$ and ${\rm r_e(\sw)} = \tr(\sw) = p$. 
\begin{enumerate}[itemsep = 1pt]
\item[(i)] By the identity $\sz^{1/2}\bZ^+ = (\sz^{-1/2}\bZ^\T \bZ \sz^{-1/2})^+ \sz^{-1/2} \bZ^\T$, we have on   $\cE(\bZ,\bX)$,  
\begin{align*}
\|\bZ^+\bW\bX^+\bY\|_{\sz} \le {2\over n}\|
\sz^{-1/2}\bZ^\T \bW 
\|_{\op} \|\bX^+\bY\|_2.
\end{align*}
Since $\|\bX^+\bY\|_2 \le \sqrt{\|\bY\|_2 / \sigma_n(\bX)} \le \sqrt{8n/p}$ and standard concentration arguments ensure that $\|
\sz^{-1/2}\bZ^\T \bW 
\|_{\op}$ is at least of order $\cO_{\PP}(\sqrt{np})$, we end up with a trivial bound $\|\bZ^+\bW\bX^+\bY\|_{\sz} = \cO_\PP(1)$. Note that since $\bX = \bZ A^\T + \bW$ depends on both $\bW$ and $\bY$, the above arguments do not appear to be loose. 
\item[(ii)] In fact, by $\bX = \bZ A^\T + \bW$ and $\bX^+ = \bX^\T (\bX\bX^\T)^+$, we also have 
\begin{align*}
\|\bZ^+\bW\bX^+\bY\|_{\sz} &= \|\bZ^+\bW(\bZ A^\T + \bW)^\T(\bX\bX^\T)^+\bY\|_{\sz}\\
&\le \|\bZ^+ \bW  A \bZ^\T (\bX\bX^\T)^+\bY\|_{\sz} + \|\bZ^+\bW\bW^\T (\bX\bX^\T)^+\bY\|_{\sz}.
\end{align*}
By similar arguments,
the second term could be bounded by 
\begin{align*}
\|\bZ^+\bW\bW^\T\|_\op  \|(\bX\bX^\T)^+\bY\|_{2} \le {2\over n}\|\sz^{-1/2}\bZ^\T\bW\bW^\T\|_\op   \sqrt{8n \over p}.
\end{align*}
As $\|\sz^{-1/2}\bZ^\T\bW\bW^\T\|_\op$ has the order of $\cO_\PP(p\sqrt{n})$, we again obtain a trivial bound.

\item[(iii)] Since $\bX = \bZ A^\T + \bW$ and especially for $p\gg n \gg K$, intuitively, we would expect that $\|\bZ^+\bW\bX^+\bY\|_{\sz}$ contains, or is mainly about, the term
\[
\|\bZ^+\bW\bW^+\bY\|_{\sz} = \|\bZ^+ \bI_n \bY\|_{\sz} = \|\bZ^+ \bY\|_{\sz}.
\]
As we have argued that $ \|\bZ^+\bY\|_{\sz}$ concentrates around $\|\beta\|_{\sz} = \pi_0\pi_1 \Dt$ in attempt (i), it seems hopeless for the rate of $\|\bZ^+\bW\bW^+\bY\|_{\sz}$   to converge to zero for non-vanishing $\Dt$.
\end{enumerate}

Despite these failed attempts, the situation can be salvaged 
via a more delicate argument. This is done by splitting $\|\bZ^+\bW\bX^+\bY\|_{\sz}$  into two parts, 
\[
\|\bZ^+\bW\bX^+\bY\|_{\sz} \le \|\bZ^+\bW V_K V_K^\T \bX^+\bY\|_{\sz} + \|\bZ^+\bW V_{-K} V_{-K}^\T \bX^+\bY\|_{\sz}
\]
based on $V_K$, the first $K$   right-singular vectors of $\bX$, and  $V_{-K}$, the last $(p-K)$ right-singular vectors of $\bX$.
The key is to recognize and capture the implicit regularization of $\wh\theta = \bX^+\bY$ in the second term in the high-dimensional regime (see, Sections \ref{app_sec_part_II} and \ref{app_sec_part_III}). This is highly nontrivial even in the ideal case where $V_KV_K^\T$ is close to the projection onto the column space of $A$. 
We use a key observation made in \cite{Bai-factor-model-03} that $\bX V_{-K}$ estimates $\bZ Q$ well for a certain $K\times K$ transformation matrix $Q$. We sharpen this result in Lemmas \ref{lem_D_H} \& \ref{lem_Z_diff} by relaxing the stringent condition $\lambda_1(A\sz A^\T) \asymp \lambda_K(A\sz A^\T) \asymp p$  imposed by  \cite{Bai-factor-model-03}.

\section{Simulation Study}\label{sec_sim}

In this section we first verify the inconsistency of the naive classifier that uses the naive plug-in estimator of $\beta_0$ and contrast with other consistent classifiers. We then evaluate the performance of our propose classifier in terms of its misclassification error as well as its estimation errors of $\beta$ and $\beta_0$. We also examine their dependence on the dimensions $p$ and $K$ as well as the signal-to-noise ratio $\xi$.

We generated the data as follows:  We set $\pi_0 = \pi_1 = 0.5$, $\alpha_0 = -\alpha_1$, $\alpha_1  = \b1_K\sqrt{2/K}$ and $\szy = \bI_K$ such that $\Dt^2 = 8$. The entries of  $\bW$ and $A$ are independent  realizations of $N(0,1)$  and $N(0, 0.3^2)$, respectively.

\subsection*{Inconsistency of the Naive Classifier} 
We refer as GLS-Naive the classifier $\wh g(x)= \1\{ x^\T \wh \theta + \wh \beta_0 > 0\}$ with $\wh\beta_0$ being the naive plug-in estimator in (\ref{def_theta_0_hat}),  while  GLS-Oracle, GLS-Plugin and GLS-ERM represent the classifiers 
$\1\{ x^\T \wh \theta + \bar \beta_0 > 0\}$
with $\bar \beta_0$ chosen as 
the true $\beta_0$,
the plug-in estimate (\ref{def_theta_0_hat2}) based on data splitting, and 
the estimate (\ref{erm}) based on empirical risk minimization in Remark \ref{rem_erm}, respectively. 
Besides the optimal Bayes classifier (Bayes), we also choose the oracle procedure (Oracle-LS) that 
uses both $\bZ$ and $Z$ to estimate $\beta$ and  $\beta_0$ in (\ref{Bayes_rule}) as our benchmark.

In the left panel of Figure \ref{fig_error_gls}, we plot the performance of all classifiers on 200 test data points  by fixing $K=5$ and  $n = 100$, while varying $p\in\{300, 600, 1000, 2000, 4000, 6000\}$. Each setting is repeated 100 times and the averaged results are reported. For GLS-Plugin and GLS-ERM, we additionally generate $100$ data points as the validation set. Clearly, GLS-Naive is inconsistent while the other three GLS-based classifiers 
get closer to the Oracle-LS as $p$ increases. Moreover, the performance of GLS-Plugin is as good as GLS-Oracle (that uses the true $\beta_0$) and better than GLS-ERM. 

We also plot the training misclassification errors of all classifiers in the right panel of Figure \ref{fig_error_gls}. As expected from Proposition \ref{lem_GLS}, GLS-Naive interpolates the training data despite its inconsistency. As discussed after Lemma \ref{lem_beta_0_sign}, by recalling that $\pi_0 = 1/2$ hence $\beta_0 = 0$, GLS-Oracle also interpolates the training data. On the other hand, neither GLS-Plugin nor GLS-ERM interpolates. This is because their estimates of $\beta_0$ are centered around zero and only the non-positive ones lead to interpolation according to our discussion after Lemma \ref{lem_beta_0_sign}. Furthermore, if we encode $Y\in \{-1,1\}$, simulation shows that both GLS-Plugin and GLS-ERM also interpolate the training data in addition to GLS-Oracle and GLS-Naive.

\begin{figure}[ht]
\centering
\includegraphics[width = .48\textwidth]{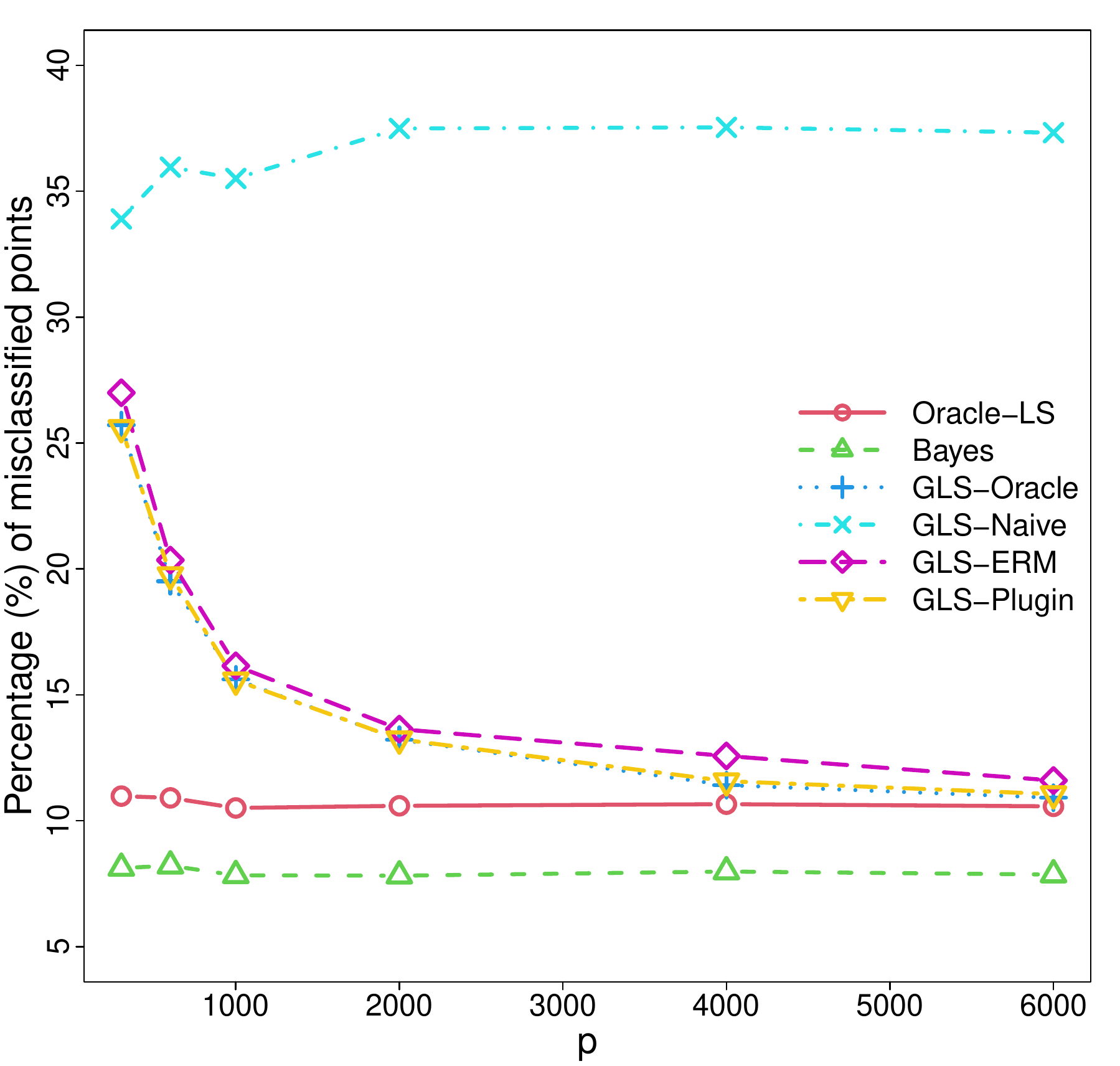}
\includegraphics[width = .48\textwidth]{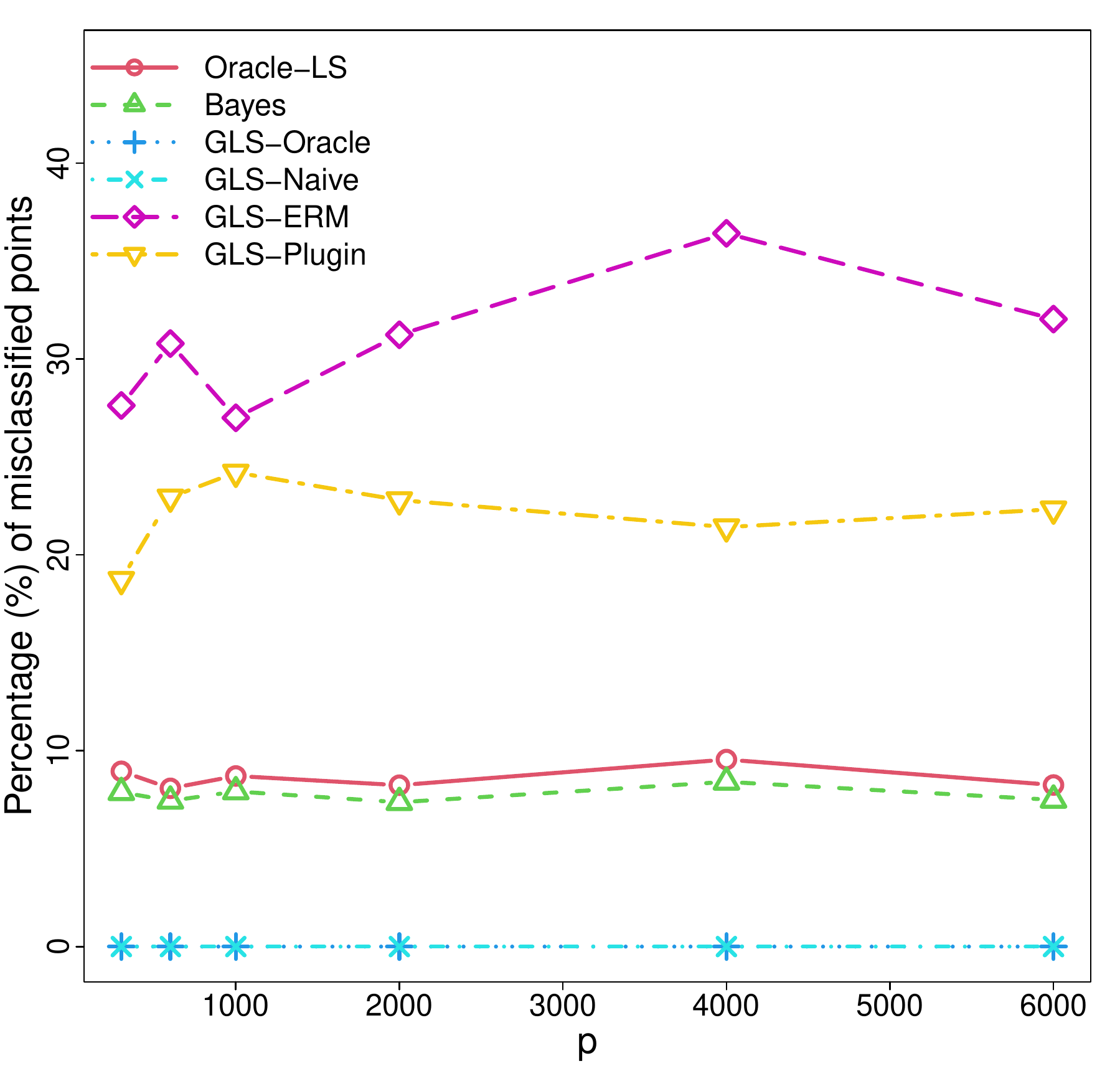}
\caption{The averaged misclassification errors of each algorithm for various choices of  $p$. The left panel depicts the misclassification errors of the training data while the right one shows the test misclassification errors.}  
\label{fig_error_gls}
\end{figure}

\subsection*{Performance of the Proposed Classifier}

We evaluate the performance of our proposed classifier, GLS-Plugin, and examine its dependence on $p$, $K$ and $\xi$ by varying them one at a time. We consider three metrics: the misclassification error on $200$ test data points, the estimation error of $\beta$, $\|\beta-A^\T\wh\theta\|_{\sz}$, as analyzed in Theorem \ref{thm_rate_Atheta}, and the estimation error of $\beta_0$, $|\wt\beta_0 -\beta_0|$. The sample size is fixed as $n = 100$ and we use a validation set with $100$ data points to compute $\wt \beta_0$. To vary the signal-to-noise ratio $\xi$, we choose the standard deviation $\sigma_A$ of each entries of $A$ from $\{0.01, 0.05, 0.1, 0.2\}$. Note that a larger $\sigma_A$ implies a larger $\xi$. 

We repeat each setting 100 times and the averaged metrics as well as the standard errors are reported  in Table \ref{tab_error}. In line with Theorem \ref{thm_rate_Atheta}, Proposition \ref{prop_intercept} and Theorem \ref{thm:rates},
all three metrics decrease in $p$ and $\xi$, while they increase in $K$.

\begin{table}[ht]
\centering
\caption{The averaged metrics of GLS-Plugin over 100 repetitions (the numbers within parentheses are the standard errors).}
\label{tab_error}
{\renewcommand{\arraystretch}{1.3}
\resizebox{\textwidth}{!}{
\begin{tabular}{lccc}
	\hline
	Setting	& Misclassification errors & Errors of estimating $\beta$ &  Errors of estimating $\beta_0$   \\\hline  
	\multicolumn{1}{l}{
		$K=5$, $\sigma_A = 0.3$}\\
	$p=300$ & 0.256 (0.046) & 0.144 (0.052) & 0.040 (0.031) \\ 
	$p=600$ & 0.198 (0.037) & 0.127 (0.046) & 0.034 (0.023) \\ 
	$p=1000$ & 0.156 (0.032) & 0.117 (0.041) & 0.029 (0.021) \\ 
	$p=2000$ & 0.132 (0.034) & 0.115 (0.039) & 0.029 (0.024) \\ 
	$p=4000$ & 0.116 (0.027) & 0.112 (0.032) & 0.027 (0.020) \\ \hline
	\multicolumn{4}{l}{$p = 1000$, $\sigma_A = 0.3$}\\
	$K=3$ & 0.152 (0.033) & 0.091 (0.039) & 0.028 (0.020) \\ 
	$K=5$ & 0.161 (0.029) & 0.117 (0.039) & 0.032  (0.022) \\ 
	$K=10$ & 0.178 (0.036) & 0.180 (0.036) & 0.033 (0.027) \\ 
	$K=15$ & 0.186 (0.038) & 0.219 (0.040) & 0.030 (0.022) \\\hline 
	\multicolumn{4}{l}{$p = 1000$, $K = 5$}\\
	$\sigma_A=0.01$ & 0.479 (0.038) & 0.397 (0.004) & 0.048 (0.039) \\ 
	$\sigma_A=0.05$ & 0.282 (0.039) & 0.239 (0.024) & 0.034 (0.026) \\ 
	$\sigma_A=0.1$ & 0.187 (0.035) & 0.124 (0.037) & 0.029 (0.019) \\ 
	$\sigma_A=0.24$ & 0.161 (0.033) & 0.109 (0.034) & 0.029 (0.022)\\\hline 
\end{tabular}
}}
\end{table}

\section{Main proofs}\label{sec_proofs}

\subsection{Proof of Theorem \ref{thm_rate_Atheta}} \label{sec:proof:main}
Let $\bX = \bZ A^\T  + \bW$ be the matrix version of model (\ref{model_X}) based on independent observations and write the singular value decomposition  of $(np)^{-1/2} \bX$ as 
\begin{align}\label{svd}
( {np})^{-1/2} \bX = \sum_{k=1}^{n} d_k u_k v_k^\T = U D V^\T= U_KD_KV_K^\T + U_{-K}D_{-K}V_{-K}^\T
\end{align}
with $D = (D_K; D_{-K})=  \diag(d_1, \ldots, d_n)$,
$D_K=  \diag(d_1, \ldots, d_K)$, $D_{-K}=  \diag(d_{K+1}, \ldots, d_n)$,
$U = (U_K; U_{-K})= (u_1,\ldots, u_n)$,
$U_K = (u_1, \ldots, u_K)$, $U_{-K} = (u_{K+1}, \ldots, u_n)$,
and $V = (V_K; V_{-K})= (v_1, \ldots, v_n)$,
$V_K = (v_1, \ldots, v_K)$, $V_{-K} = (v_{K+1}, \ldots, v_n)$.	
Define	 
\begin{equation}\label{def_err_w}
\errw := \|\sw\|_\op \left( 1  + {{\rm r_e}(\sw) \over n}\right).
\end{equation}
We  define  the events
\begin{align*}
&\cE_Z := \left\{
\frac12 \le {1\over n}\sigma_K^2(\bZ \sz^{-1/2})  \le {1\over n}\sigma_1^2(\bZ  \sz^{-1/2}) \le 2
\right\}\\
&\cE_W := \left\{
{1\over n}\sigma_1^2(\bW) \le 12\sigma^2 \errw
\right\}\\
&\cE_X := \left\{\sigma_n^2(\bX) \ge \frac{1}{8}  \tr(\sw) \right\}
\end{align*}
In the sequel, we work on the event 
\[ \cE := \cE_Z \cap \cE_W \cap \cE_X.\] This event 
holds with probability greater than $1- \cO(n^{-c})$, see Proposition \ref{lem_lb_sigma_n_X} and Lemmas \ref{lem_Z} and \ref{lem_op_norm}.
Observe that we need at this point our assumptions  $n/{\rm r_e}(\sw)\to0$ and $K\log(n)/n\to0$ in (\ref{assumps}).

Since $\bZ^+\bZ = \bI_K$ and $\bX\bX^+ = \bI_n$ on the event $\cE_Z \cap \cE_X$, we find 
\begin{align}\label{decomp1}
A^\T \wh\theta  = \bZ^+ \bZ A^\T \bX^+ \bY = \bZ^+ (\bX - \bW)\bX^+ \bY  = \bZ^+\bY - \bZ^+\bW\bX^+\bY
\end{align}
so that 
\begin{align}\label{decomp2}
\|\sz^{1/2}(A^\T \wh\theta - \beta)\|_2 &\le \|\sz^{1/2}(\bZ^+\bY - \beta)\|_2 + \|\sz^{1/2}\bZ^+\bW\bX^+\bY\|_2\\
&\le  \|\sz^{1/2}(\bZ^+\bY - \beta)\|_2 +\|\sz^{1/2}\bZ^+\bW V_KV_K^\T \bX^+\bY\|_2\nonumber\\
&\quad +  \|\sz^{1/2}\bZ^+\bW  V_{-K}V_{-K}^\T\bX^+\bY\|_2.\nonumber
\end{align} 
We bound
the first term $\|\sz^{1/2}(\bZ^+\bY - \beta)\|_2$ in (\ref{decomp2})   by $C\sqrt{K\log(n)/n}$ in  Lemma \ref{lem_ZY_beta} below. The third term,
$\|\sz^{1/2}\bZ^+\bW  V_{-K}V_{-K}^\T\bX^+\bY\|_2$, turns out to be more difficult to bound and we analyze it separately in the next section \ref{app_sec_part_II}. 
We first bound  the second term $\|\sz^{1/2}\bZ^+\bW V_KV_K^\T \bX^+\bY\|_2$ in (\ref{decomp2}).  Let  $V_A\in \cO_{p\times K}$ be the matrix with   the left-singular vectors of $A$ as its columns.  We have
\begin{align*}
&\|\sz^{1/2}\bZ^+\bW V_KV_K^\T \bX^+\bY\|_2\\
& ={1\over \sqrt{np}} \|\sz^{1/2}\bZ^+\bW V_KV_K^\T V_KD_K^{-1}U_K^\T\bY\|_2 \\
&\le \frac{2}{n\sqrt{p}}  \| \bW V_KV_K^\T V_K \|_\op \|D_K^{-1}U_K^\T\bY\|_2 \\
& \le \frac{2}{n\sqrt{p}}  \left(
\| \bW V_A V_A^\T V_K\|_\op +
\| \bW (V_A V_A^\T - V_KV_K^\T) V_K\|_\op 
\right) \|D_K^{-1}U_K^\T\bY\|_2
\\
&\le  {2\over n\sqrt{p}}\left( \|\bW V_A\|_\op +  \|\bW\|_\op  \|V_KV_K^\T - V_AV_A^\T\|_\op\right) { \|D_K^{-1}U_K^\T\bY\|_2   }.
\end{align*}
In first inequality, we used $\|\sz^{1/2}\bZ^+\|_\op = 1 / \sigma_K(\bZ\sz^{-1/2}) \le 2$ on $\cE_Z$. 	The last inequality uses  $\| V_A\|_\op\le1$ and $\| V_K\|_\op\le 1$ since   $V_K, V_A\in \cO_{p\times K}$.   Lemma \ref{lem_D_H} ensures
\[
{ \|D_K^{-1}U_K^\T\bY\|_2   \over \sqrt{np}} \le  {\|\bY\|_2 \over d_K\sqrt{np}} \le \frac{1}{d_k\sqrt{p}} \asymp \sqrt{1\over \lambda_K}.
\]
Note that we need $\lambda_K\gg \kappa \delta_W$ in (\ref{assumps}) in order to apply this lemma.
Invoke Lemma \ref{lem_W} and Lemma \ref{lem_PA_diff}
(that require $K\log(n) \ll n$ and $\delta_W\ll \lambda_K$), 
and  conclude 
\begin{equation}\label{bd_part1}
\|\sz^{1/2}\bZ^+\bW V_KV_K^\T \bX^+\bY\|_2 \lesssim
\sqrt{\|\sw\|_\op \over \lambda_K}+
{\errw \over \lambda_K} \sqrt{\kappa}.
\end{equation}
The proof of Theorem \ref{thm_rate_Atheta} is completed by collecting the bounds in Lemma \ref{lem_ZY_beta}, inequality (\ref{bd_part1}) and inequality (\ref{ineq:difficult}) in Appendix \ref{app_sec_part_II}.
\qed

\subsubsection{Bound of $\|\sz^{1/2}\bZ^+\bW  V_{-K}V_{-K}^\T\bX^+\bY\|_2$}\label{app_sec_part_II}
We will prove that, for any $c>0$, there exists a finite $C>0$, independent of $n$, such that
\begin{align}\label{ineq:difficult}
C^{-1}	\|\sz^{1/2}\bZ^+\bW  V_{-K}V_{-K}^\T\bX^+\bY\|_2 &\le   {\kappa \sqrt{nK\log n}  \over {\rm r_e}(\sw)}+   \sqrt{n\over {\rm r_e}(\sw)} +  \kappa\sqrt{\|\sw\|_\op \over \lambda_K}
\end{align}
holds with probability  larger than $1-\cO(n^{-c})$.	
By  $\bX^+ = \bX^\T (\bX\bX^\T)^+$ and the definition of the event $\cE_Z$, we have
\begin{align*}
\|\sz^{1/2}\bZ^+\bW  V_{-K}V_{-K}^\T\bX^+\bY\|_2& \le 
{2\over n}\|\sz^{-1/2}\bZ^\T \bW V_{-K}V_{-K}^\T \bX^\T (\bX \bX^\T)^+ \bY\|_2\\
&= {2\over n}\|\sz^{-1/2}\bZ^\T \bW \bX^\T U_{-K}U_{-K}^\T (\bX \bX^\T)^+ \bY\|_2\\
&\le \rI + \rII 
\end{align*}
on the event $\cE_Z$,
where 
\begin{align*}
\rI & = {2\over n}\|\sz^{-1/2}\bZ^\T \bW A \bZ^\T U_{-K}U_{-K}^\T (\bX \bX^\T)^+ \bY\|_2,\\
\rII &= {2\over n}\|\sz^{-1/2}\bZ^\T \bW \bW^\T U_{-K}U_{-K}^\T (\bX \bX^\T)^+ \bY\|_2.
\end{align*}
We bound $\rI$ and $\rII$ in the sequel separately. Following \cite{Bai-factor-model-03}, we define
\[
\wh \bZ = \sqrt{n} ~  U_K \in \RR^{n\times K}
\]
and 
\[
H = {1\over np}A^\T A \bZ^\T \wh\bZ D_K^{-2} \in  \RR^{K\times K},
\]
such that $\bZ H  = (np)^{-1}\bZ A^\T A \bZ^\T \wh\bZ D_K^{-2}$ is invariant to different parametrizations of $A$ and $\szy$.
Note that, on the event $\cE_Z\cap \cE_W$ and under the set of assumptions (\ref{assumps}),  the results  of Lemma \ref{lem_D_H} hold, and, in particular, 
\begin{equation}\label{bd_H}
\sigma_K\left(\sz^{1/2}H\right) \gtrsim \sqrt{\lambda_K \over \lambda_1},\qquad \sigma_1\left(H^{-1} A^\T\right) \lesssim \sqrt{\lambda_1}.
\end{equation}

\paragraph{Bound of $\rI$:}

Notice that 
\begin{align}\label{bd_een}
\rI  &\le {1\over n}\|\sz^{-1/2}\bZ^\T \bW A (H^{\T})^{-1} \|_\op  \|H^\T  \bZ^\T U_{-K}  U_{-K}^\T (\bX \bX^\T)^+ \bY\|_2.
\end{align}
We find that, on the event $\cE_X$, 
\begin{align}\label{bd_ZUXY}\nonumber
& \|H^\T  \bZ^\T U_{-K}  U_{-K}^\T (\bX \bX^\T)^+ \bY\|_2\\  &= 	\|(\bZ H - \wh\bZ)^\T U_{-K}  U_{-K}^\T (\bX \bX^\T)^+ \bY\|_2 &&\text{since $\wh\bZ ^\T = \sqrt{n} U_{K}$} \nonumber \\\nonumber
&\le 	\|\wh\bZ - \bZ H  \|_\op \| (\bX \bX^\T)^+ \bY\|_2 && \text{since $ U_{-K} U_{-K}^\T$ is a projection} \nonumber\\
&\le    \|\wh\bZ - \bZ  H \|_\op {\|\bY\|_2 \over \sigma_n^2(\bX)} \\ 
&\le 8 \|\wh\bZ - \bZ  H  \|_\op {\sqrt n  \over \tr(\sw)} && \text{on $\cE_X$} \nonumber \\
&\lesssim  \sqrt{ \|\sw\|_\op \over \lambda_K} \sqrt{\lambda_1 \over \lambda_K} { n  \over \tr(\sw)} &&\text{by Lemma \ref{lem_Z_diff}}\nonumber
\end{align} 
Lemma  \ref{lem_ZWVA} and (\ref{bd_H}) ensure that
\begin{align}\label{bd_twee}
{1\over n}\|\sz^{-1/2}\bZ^\T \bW A (H^{\T})^{-1} \|_\op  &~ \le  ~  {1\over n}\|\sz^{-1/2}\bZ^\T \bW V_A\|_\op \|A  (H^{\T})^{-1}\|_\op\nonumber\\
& ~ = ~  {1\over n}\|\sz^{-1/2}\bZ^\T \bW V_A\|_\op \|A  (H^{-1})^\T\|_\op \\
& ~ \le  ~ C \sqrt{K\log n \over n}\sqrt{\lambda_1 \|\sw\|_\op},\nonumber
\end{align}
with probability $1-\cO(n^{-c})$. 
From (\ref{bd_een}), (\ref{bd_ZUXY}) and (\ref{bd_twee}),
we conclude 
\begin{equation}\label{bd_I}
\rI\le C {\lambda_1 \over \lambda_K}{ \|\sw\|_\op \over \tr(\sw)} \sqrt{nK\log n} = C {\kappa \over {\rm r_e}(\sw)}\sqrt{nK\log n}
\end{equation}
with probability $1-\cO(n^{-c})$. 

\paragraph{Bound of $\rII$:} By adding and subtracting $\tr(\sw)\bI_n$, we have $\rII \le \rII_1 + \rII_2$ where 
\begin{align*}
\rII_1 &= {1\over n}\|\sz^{-1/2}\bZ^\T (\bW \bW^\T - \tr(\sw)\bI_n) U_{-K}U_{-K}^\T (\bX \bX^\T)^+ \bY\|_2,\\
\rII_2 &=  { \tr(\sw)\over n}\|\sz^{-1/2}\bZ^\T  U_{-K}U_{-K}^\T (\bX \bX^\T)^+ \bY\|_2.
\end{align*}
Next, we invoke the definitions of the events $\cE_X$, $\cE_Z$, and   apply Lemma \ref{lem_W} to obtain
\begin{align}\label{bd_rII_1}
\rII_1 &\lesssim {1\over \sqrt n}\|\bW \bW^\T - \tr(\sw)\bI_n\|_\op{\|\bY\|_2 \over \sigma_n^2(\bX)} \le C \sqrt{ n \|\sw\|_\op \over \tr(\sw)} = C \sqrt{n\over {\rm r_e}(\sw)}
\end{align} 
with probability $1-\cO(n^{-c})$.
Regarding $\rII_2$,  we observe that 
\begin{align}\label{bd_rII_2}
\rII_2  &\le    { \tr(\sw)\over n}\|\sz^{-1/2} (H^{\T})^{{-1}}\|_\op \|H^\T\bZ ^\T  U_{-K}U_{-K}^\T (\bX \bX^\T)^+ \bY\|_2\nonumber\\
&\lesssim \frac{ \tr(\sw)}{n} \sqrt{ \lambda_1\over\lambda_K} \sqrt{ \|\sw\|_\op \over\lambda_K} \sqrt{\lambda_1\over \lambda_K} {n \over \tr(\sw)} &&\text{ using (\ref{bd_H}) and  (\ref{bd_ZUXY})} \nonumber\\
&=
\kappa\sqrt{\|\sw\|_\op \over \lambda_K}.
\end{align} 
Collecting the bounds in  (\ref{bd_I}), (\ref{bd_rII_1}) and (\ref{bd_rII_2}) yields the desired result. \qed

\subsubsection{Technical lemmas used in the proof of Theorem \ref{thm_rate_Atheta}}\label{app_sec_part_III}

\begin{lemma}\label{lem_W}
Assume $K\le n$.
With probability $1-\exp(-n)$,  we have 
\begin{align*}
{1\over n}\sigma_1^2(\bW V_A) \le 14 \sigma^2  \|\sw\|_\op
\end{align*}
Assume
${\rm r_e}(\sw)\ge n$. With probability $1-2\exp(-n)$, we have for some positive, universal constant $C$,  
\begin{align*}
\left\|\bW\bW^\T - \tr(\sw)  \bI_n\right\|_\op &\le C  \sigma^2\sqrt{n\|\sw\|_\op \tr(\sw)}.
\end{align*} 
\end{lemma}
\begin{proof}
To prove the fist result, recall that
$\bW = \wt \bW \sw^{1/2}$.  An application of Lemma \ref{lem_op_norm} gives
\[
\PP\left\{{1\over n}\| \bW V_AV_A^\T \bW^\T \|_{{\rm op}} \le \sigma^2\left( \sqrt{{\rm tr}(M) \over n} + \sqrt{6\|M\|_{\op}}
\right)^2\right\} \ge  1 -  \exp(-n).
\] 
Here $M = \sw^{1/2}V_AV_A^\T \sw^{1/2}$. The first claim now follows from
$
\tr(M) \le K \|M\|_\op$, $\| M\|_\op \le \|\sw\|_\op
$,  the inequality $(x+y)^2 \le 2x^2 + 2y^2$
and our assumption $K \le  n$. 

For the proof of the second claim, we use a standard discretization argument to obtain
\begin{align*}
\left\|\bW\bW^\T - \tr(\sw)  \bI_n\right\|_\op &= \sup_{u \in \cS^{n-1}} u^\T \left(\bW\bW^\T - \tr(\sw)  \bI_n\right) u\\
&\le 2 \max_{u\in \cN_n(1/4)}u^\T \left(\bW\bW^\T - \tr(\sw)  \bI_n\right) u
\end{align*}
Here $\cN_n(1/4)$ is a minimal $(1/4)$-net  of $\cS^{n-1}$.  It has cardinality $|\cN_n(1/4)|\le 9^n$ (see, for instance, Lemma 5.4 of \cite{vershynin_2012}).
For any fixed $u\in \cN_n(1/4)$, we apply  the Hanson-Wright inequality (see, \cite{rudelson2013hanson}) to find, for any $t\ge 0$, 
\begin{align*}
\PP\left\{
\left|u^\T \left(\bW\bW^\T - \tr(\sw)  \bI_n\right) u \right| > t
\right\} & \le 2\exp \left\{
-c \min \left(
{t^2 \over \sigma^4 \tr(\sw^2)}, ~ {t\over \sigma^2 \|\sw\|_\op}
\right)
\right\}\\
&\le 2\exp \left\{
-c  \min \left(
{t^2\over \sigma^4\|\sw\|_\op \tr(\sw)}, ~{t\over \sigma^2 \|\sw\|_\op}
\right)
\right\}.
\end{align*}
Here $c\le 1$ is some universal constant.
Next, we 
choose 
\[ t =  C  \sigma^2 \sqrt{n \|\sw\|_\op \tr(\sw)}\] with $C= \log(9e)/ {c} \ge 1$ and
we take a union bound over $u\in \cN_n(1/4)$ to conclude
\begin{align*}
\left\|\bW\bW^\T - \tr(\sw)  \bI_n\right\|_\op \le \frac{ \log(9e)}{c}  \sigma^2  \sqrt{n \|\sw\|_\op \tr(\sw)}
\end{align*}
with probability at least 
\begin{align*}
1 - 2|\cN_{1/4}|\exp \left\{
-  c  \min \left( nC^2
, ~ C \sqrt{n \tr(\sw) \over \|\sw\|_\op}
\right)
\right\} &\ge 1 - 2 \cdot 9^n \exp({-c C n})\\
&\ge 1- 2\exp(-n)
\end{align*}
We used our assumption ${\rm r_e}(\sw)\ge n$ in the first inequality.  
\end{proof}

\bigskip

The following lemma states the rates of the first $K$ singular values of $\bX$ and provides a lower bound for $\sigma_K(\sz^{1/2} H)$.
\begin{lemma}\label{lem_D_H}
Assume $\lambda_K \ge 48 \sigma^2 \kappa \errw  $.
On the event $\cE_Z \cap \cE_W$, we have 
\begin{align*}
&	\frac12	\sqrt{\lambda_k / p}\le 		d_k \le 4 \sqrt{\lambda_k / p} \qquad \forall k\in [K]
\\		
&
\sigma_K^2 \left(\sz^{1/2}H\right) \ge  { \frac{ \lambda_K }{2 \lambda_1}}= {1\over 2\kappa}
\\
& \sigma_1^2\left(H^{-1} A^\T\right) \le 4{\lambda_1}
\end{align*}
\end{lemma}
\begin{proof}
We work on $\cE_Z \cap \cE_W$. For the first claim.
For any $k\in [K]$, we have
\begin{align*}
d_k = {1\over \sqrt{np}} \sigma_k(\bX) &\ge {1\over \sqrt{np}}\left[
\sigma_k(\bZ A^\T) - \sigma_1(\bW)
\right]&&\textrm{using Weyl's inequality }\\
&\ge 
\sqrt{1 \over2 p}\sigma_k(\sz^{1/2}A^\T) -  \sqrt{12 \sigma^2\errw \over p}
&&\textrm{on }\cE_Z\cap \cE_W\\
&\ge \frac12 \sqrt{\lambda_k\over p} &&\textrm{since }\lambda_K \ge  48 \sigma^2\errw.
\end{align*}
Similarly, we also have 
\[
d_k \le {1\over \sqrt{np}}\left[
\sigma_k(\bZ A^\T) + \sigma_1(\bW)
\right] \le 4  \sqrt{\lambda_k\over p}.
\]
We bound from below $\sigma_K(H)$ as follows:
\begin{align*}
\sigma_K^2(\sz^{1/2}H) &= {1\over n^2p^2}\lambda_K\left(
D_K^{-2}\wh\bZ^\T \bZ\sz^{-1/2}\left(\sz^{1/2}A^\T A\sz^{1/2}\right)^2 \sz^{-1/2}\bZ^\T \wh\bZ D_K^{-2}
\right)\\
& \ge {\lambda_K \over  n^2p^2}\lambda_K\left(
D_K^{-2}\wh\bZ^\T \bZ  A^\T A  \bZ^\T \wh\bZ D_K^{-2}
\right)\\
&={\lambda_K \over  n^2p^2}\sigma_K^2(D_K^{-2}\wh\bZ^\T \bZ A^\T).
\end{align*}
Since 
\begin{align}\label{lb_sig_DZZA}\nonumber
\sigma_K(D_K^{-2}\wh\bZ^\T \bZ A^\T)  &\ge \sigma_K(D_K^{-2}\wh\bZ^\T \bX) - \sigma_1(D_K^{-2}\wh\bZ^\T \bW) \ &&\text{by Weyl's inequality}\\ \nonumber
&\ge  n\sqrt{p} ~ \sigma_K(D_K^{-1}) - \sigma_1(\bW) \sigma_1(D_K^{-2}) \sigma_1(\wh\bZ) && \textrm{by }\wh \bZ=\sqrt{n}~ U_K\\\nonumber
&\ge {np \over \sqrt{\lambda_1}} - {np}\sqrt{12\sigma^2 \errw \over \lambda_K^2} \ &&\text{since $\sigma_1(\wh \bZ)=\sqrt{n}$}\\
& \ge \frac12  {np  \over \sqrt{\lambda_1}} &&\textrm{by } \lambda_K^2\ge  48\sigma^2 \lambda_1 \errw.
\end{align}
We used the first result   in the last two steps.

Finally, by analogous arguments, the last result follows from 
\begin{align*}
\sigma_1^2(H^{-1}A^\T) &= \lambda_1\left(
(\sz^{1/2}H)^{-1} \sz^{1/2}A^\T A \sz^{1/2} 	(H^\T \sz^{1/2})^{-1}
\right)\\
&= \left[
\lambda_K\left(H^\T \sz^{1/2} (\sz^{1/2}A^\T A \sz^{1/2} )^{-1}  \sz^{1/2}  H \right)
\right]^{-1}\\
& = n^2p^2\left[
\lambda_K\left(D_K^{-2} \wh\bZ^\T \bZ A^\T A \bZ^\T \wh \bZ D_K^{-2} \right)
\right]^{-1}\\
&=  n^2p^2\left[
\sigma_K^2\left(D_K^{-2} \wh\bZ^\T \bZ A^\T \right)
\right]^{-1}\\
&\le 4 \lambda_1 && \textrm{by (\ref{lb_sig_DZZA})}.
\end{align*}
This completes our proof.
\end{proof}

\begin{lemma}\label{lem_Z_diff}
Assume $\lambda_K \ge 8\errw (1 \vee 6\sigma^2 \kappa )$,  $K\le  n$ and ${\rm r_e}(\sw)\ge n$.  
On the event $\cE_Z\cap \cE_W$, we have, with probability $1-3\exp(-n)$,  
\[
\| \wh \bZ - \bZ  H\|_\op \lesssim  \sqrt{n \|\sw\|_\op \over \lambda_K} \sqrt{\lambda_1 \over \lambda_K}.
\]
\end{lemma}
\begin{proof} 
We work on the event $\cE_Z \cap \cE_W$.
First, by the SVD of $\bX= UD V^\T/ \sqrt{np}$, we find  the following identity
\[
{1\over np}\bX\bX^\T \wh\bZ = UD^2U^\T U_K \sqrt{n} = \sqrt{n}U_KD_K^2 = \wh\bZ D_K^2.
\]
Further observe that 
\[
{1\over np} \left(\bX\bX^\T - \tr(\sw) \bI_n\right) \wh\bZ   = \wh \bZ \left(
D_K^2 -{\tr(\sw)\over np}\bI_K 
\right).
\]
Define the matrix 
\[
J =  D_K^2 -{\tr(\sw)\over np} \bI_K 
\]
and note that, by our assumption $\lambda_K\ge 8 \errw$,
\begin{align*}
\sigma_K(J) &\ge d_K^2 - {\tr(\sw)\over np} &&\text{by the definition of $J$}\\
&
\ge {1\over p}\left(
{\lambda_K\over 4} - {\tr(\sw)\over n}	\right)  &&\text{using Lemma \ref{lem_D_H} with   $\lambda_K \ge  48\sigma^2 \kappa \errw$}\\
&	\ge  {\lambda_K \over8 p} &&\text{ since  $\lambda_K \ge 8\errw \ge 8\tr(\sw)/n$}
\end{align*}
Plugging in $\bX = \bZ A^\T + \bW$ and rearranging terms yield 
\begin{align*}
\wh \bZ - \bZ  H = {1\over np}\left[
\bZ A^\T \bW^\T +  \bW A \bZ^\T + \left(
\bW\bW^\T - \tr(\sw) \bI_n
\right)
\right]\wh\bZ J^{-1}.
\end{align*}
The previous two displays and the inequality $\|\wh \bZ\|_\op   \le \sqrt{n}$ (as $U_K\in \cO_{p\times n}$), further imply
\[
\|\wh \bZ - \bZ H\|_\op \le {8\over \lambda_K\sqrt{n}}\left(2 \|\bZ A^\T \bW\|_\op + \left\|\bW\bW^\T - \tr(\sw)  \bI_n\right\|_\op\right).
\]
On the event $\cE_Z$,    Lemma \ref{lem_W} yields, with probability $1-3\exp(-n)$, 
\[
{1\over \sqrt n}\|\bZ A^\T \bW^\T\|_\op   \le 2 \|\bW V_A\|_\op~ \sigma_1(A \sz^{1/2}) \le 2 \sqrt{14 \sigma^2 n\lambda_1 \|\sw\|_\op}
\]
and 
\[
{1\over \sqrt n}\left\|\bW\bW^\T - \tr(\sw)  \bI_n\right\|_\op \le C \sigma^2 \sqrt{\|\sw\|_\op \tr(\sw)}.
\]
The result follows after we use $\lambda_1 \ge \lambda_K \ge \errw/8 \ge \tr(\sw)/(8n)$ and collect terms.
\end{proof}

\subsection{Proof of Proposition \ref{prop_intercept}}
\label{ss_prop_intercept}
\begin{proof}
Set $\Delta \wt\mu= (\wt\mu_0+\wt\mu_1)/2$ and  $\bar{\alpha}=(\a_0+\a_1)/2$. 
We first recall \citet[Fact 1 in Appendix C]{BW22} that, for any semi-positive definite matrix $M$, we have 
\begin{align}\label{norms}
\a_0^\T M \a_0 +  \a_1^\T M \a_1 \le \max(\pi_0,\pi_1) (\a_1-\a_0)^\T M (\a_1-\a_0)
\end{align}
and the identity \cite[(A.11) in Appendix A]{BW22}  
\begin{align}\label{norm}
(\a_1-\a_0)^\T \sz^{-1} (\a_1-\a_0) &= \frac{\Dt^2}{1+\pi_0\pi_1 \Dt^2 } \le \frac{1}{\pi_0\pi_1} 
\end{align}
which is mainly a consequence of    Woodbury's formula.
By the triangle inequality, we find 
\begin{align}\label{normm}
|\wt \beta_0 - \beta_0| &\le \left|
\bar \a^\T (A ^\T\wh \theta -\beta) +  \wt\beta_0-\beta_0
\right| +\left |\bar \a^\T (A ^\T\wh \theta -\beta) \right| \nonumber\\
&\le \left|
\bar \a^\T (A ^\T\wh \theta -\beta) +  \wt\beta_0-\beta_0
\right| + {1 \over \sqrt{\pi_0\wedge \pi_1}}  \|A ^\T\wh \theta -\beta\|_{\sz}
\end{align} 
The second inequality used
\begin{align*}
|\bar \a^\T (A ^\T\wh \theta -\beta)| &\le {1\over 2}\left(
\|\sz^{-1/2}\a_1\|  + \|{\sz^{-1/2}}\a_0\|
\right) \|A ^\T\wh \theta -\beta\|_{\sz} &&\text{by Cauchy-Schwarz}\\
&\le {1\over 2}\left(\sqrt{\pi_1} + \sqrt{\pi_0}\right)\|\a_1-\a_0\|_{\sz^{-1}} \|A ^\T\wh \theta -\beta\|_{\sz}&&\text{by (\ref{norms})}\\
&\le  {1 \over \sqrt{\pi_0\wedge \pi_1}}  \|A ^\T\wh \theta -\beta\|_{\sz}&&\text{by (\ref{norm})}.
\end{align*}
For the first term on the right of (\ref{normm}), we notice that $\bar \a^\T (A ^\T\wh \theta -\beta) +  \wt\beta_0-\beta_0 = R_1 + R_2$ with 
\begin{align*}
R_1 &=   \left(   A\bar\a    -   \Delta\wt\mu \right) ^\T \wh \theta \\
R_2 &= \left \{ 1 - (\widetilde \mu_1-\widetilde \mu_0)^\T \wh\theta \right \} \wh\pi_0\wh\pi_1 \log\frac{\wh\pi_1}{\wh \pi_0} -
\left   \{ 1 - (\a_1-\a_0)^\T \beta \right \} \pi_0\pi_1 \log\frac{\pi_1}{ \pi_0}
\end{align*}
For $R_2$, we notice, using the same tedious calculation as in \citet[Proof of Lemma 15]{BW22}, that 
\begin{align*}
|R_2| &\le   | (\wt \mu_0- \wt\mu_1)^\T \wh \theta - (\a_0-\a_1) ^\T \beta| + \frac{3}{\wh\pi_0\wedge \wh \pi_1}  
| \wh \pi_0 - \pi_0 |. 
\end{align*} 
For the first term on the right, we have
\begin{align*}
&	|(\alpha_0-\alpha_1)^\T \beta  - (\wt\mu_0 - \wt\mu_1)^\T \wh\theta| \\
& \le |(\alpha_0 - \a_1)^\T (\beta -A^\T \wh\theta)|  + | (\alpha_0-\alpha_1)^\T A^\T \wh\theta - (\wt\mu_0 - \wt\mu_1)^\T \wh\theta|\\
&\le   \|\a_1-\a_0\|_{\sz^{-1}} \| A^\T\wh\theta -\beta\|_{\sz} + | (\alpha_0-\alpha_1)^\T A^\T \wh\theta - (\wt\mu_0 - \wt\mu_1)^\T \wh\theta|\\
&\le {1\over \sqrt{\pi_0\pi_1}} \| A^\T\wh\theta -\beta\|_{\sz} + |R_1|.
\end{align*}
Hence, 		
\begin{align} \label{new1}
\left|  \bar \a^\T (A ^\T\wh \theta -\beta) +  \wt\beta_0-\beta_0 \right| 
& \le  
{\| A^\T\wh\theta -\beta\|_{\sz}\over \sqrt{\pi_0\pi_1}} + 2\left| \left(   A\bar\a    -   \Delta\wt\mu \right) ^\T \wh \theta  \right|  + {3| \wh \pi_0 - \pi_0 |\over \wh \pi_0 \wedge\wh \pi_1}.
\end{align}		
We have bounded the first term $ \| A^\T\wh\theta -\beta\|_{\sz}$ in Theorem \ref{thm_rate_Atheta} above. By Hoeffding's inequality,
$ \wh\pi_k \ge \pi_k/2$ with probability larger than $ 1- \exp(-n\pi_k^2 / 8)$, for $k\in \{0,1\}$. By further applying Hoeffding's inequality, we bound
the third term on the right in (\ref{new1}) by 
\begin{eqnarray}\label{new2} 
{6\over \pi_0\wedge \pi_1} \frac{2t}{\sqrt{n}}
\end{eqnarray}
with probability $1-2\exp(-t^2/2)-2\exp(-n (\pi_0\wedge \pi_1)^2/8)$. 

We bound the second term	on the right in (\ref{new1})	as follows.
For each $k\in \{0,1\}$,
\begin{equation*}\label{def_W_bar}
\wt \mu_k = {1\over \wt n_k}\sum_{i=1}^{n'} X_i' \1\{ Y_i' = k\} = {1\over\wt  n_k}\sum_{i=1}^{n'} (AZ_i' +W_i')\1\{Y_i' = k\} := A\wt \a_k +  \bar W_{(k)}^\prime,
\end{equation*} hence
we can write
\begin{align*}
\left| (A\a_k - \wt \mu_k )^\T   \wh\theta\, \right|  
&\le 	\left|(\a_k - \wt \a_k)^\T \beta \right| + \left|  (\wt\alpha_k-\a_k) ^\T( A^\T\wh\theta-\beta) \right| + \left| (\bar W_{(k)}')^\T \wh\theta\, \right|.
\end{align*}
Next, we use the normal assumption (iv),  the independence of $\wt \a_k$ and $\wh \theta$ and the fact that
$\sz- \szy = \pi_0\pi_1(\a_1-\a_0)(\a_1-\a_0)^\T$ is positive definite to deduce that the first two terms are bounded by 
\[ \frac{t} {\sqrt{\wt n_k}}    \left({  {\|\beta\|_{\sz} +  \| A^\T \wh \theta - \beta \|_{\sz} } }   \right)
\]
with probability $1-4\exp(-t^2/2)$.
Since 	$ \bar W_{(k)}' $ is subGaussian with parameter $\sigma (\wh\theta ^\T \sw\wh \theta / \wt n_k)^{1/2}$
and is independent of $\wh \theta$ and the labels $Y_1',\ldots,Y_{n'}'$, we have 
\[ \left| (\bar W_{(k)}')^\T \wh\theta\, \right| \le t \sigma \| \wh \theta\|_{\sw} / \sqrt{{  \wt n_k}}.\]
with probability $1-2\exp(-t^2/2)$. 
Now use the inequality 
\begin{align*}
\| \beta\|_{\sz}^2 =  ({\pi_0\pi_1})^2 \| \a_1-\a_0\|_{\sz^{-1}} ^2 \le {\pi_0\pi_1}\le 1/4,
\end{align*}
using  (\ref{norm}), 
and the bound
\begin{align*}
\| \wh \theta\, \|^2_{\sw} & \le \|\sw\|_\op \| \bX^+ \bY \|^2 \le  \|\sw\|_\op \frac{ \| \bY\|^2 }{ \sigma_n^2 (\bX) } \le \frac{8 n}{ {\rm r_e}(\sw)} \to0
\end{align*}
The last inequality  holds with probability at least $1- 3\exp(-C'n)$, by Proposition \ref{prop_hattheta}, for some $C'>0$. By Hoeffding's inequality,
$ \wt n_k \le n' \pi_k/2$ with probability larger than $ 1- \exp(-n'\pi_k^2 / 8)$. We conclude that for any $c>0$,  the combination of above bounds with $n'\asymp n$, $t=C (\log n)^{1/2}$ for  some finite $C=C(c,\pi_0,\pi_1,\sigma)$ large enough, yields  
\begin{align}
\label{new3}
\PP\left\{ 	\left| (A\a_k - \wt \mu_k )^\T   \wh\theta\, \right|  \ge C \sqrt{ \log n\over n} \left(1 + \|A^\T \wh\theta -\beta\|_{\sz}\right)  \right\} \lesssim n^{-c}
\end{align}
Finally, (\ref{normm}), (\ref{new1}), (\ref{new2}) and (\ref{new3}) prove 
our result.
\end{proof}

\section{Auxiliary lemmas}\label{app_aux}

We restate the following lemmas which are proved in \cite{BW22}. 

\begin{lemma}\cite[Lemma 31]{BW22}\label{lem_Z}
Under assumptions {\rm (ii)} and {\rm (iv)} and  $K\log n\ll   n$,  for any constant $c>0$, 
\[
\PP\left\{		\frac12 \le {1\over n}\sigma_K^2(\bZ \sz^{-1/2})  \le {1\over n}\sigma_1^2(\bZ  \sz^{-1/2}) \le 2\right\} = 1- \cO(n^{-c}).
\]
\end{lemma}

\begin{lemma}\cite[Lemma 32]{BW22}\label{lem_ZWVA}
Under assumptions {\rm (i) -- (v)}, for any $c>0$, there exists a 
$C<\infty$ such that
\begin{align*}
&\PP\left\{
{1\over n}\|\sz^{-1/2} \bZ^\T \bW V_A\|_{\op} \le C\sqrt{\|\sw\|_\op} \sqrt{K\log n\over n} 
\right\} = 1-\cO(n^{-c}).
\end{align*}
\end{lemma}

\begin{lemma}\cite[Lemma 18]{BW22}\label{lem_ZY_beta}
Under assumptions {\rm (ii)} and {\rm (iv)} and $K\log n\ll   n$, for any $c>0$, there exists a 
$C<\infty$ such that
\[
\PP\left\{\left\|\bZ^+ \bY - \beta\right\|_{\sz}\le C \sqrt{K\log n\over n}\right\} = 1-\cO(n^{-c}).
\]
\end{lemma}

\begin{lemma}\cite[Lemma 21]{BW22}\label{lem_PA_diff}
Under assumptions {\rm (i) -- (v)},	 $K\log n\ll  n$ and $\errw \ll \lambda_K$, for any $c>0$, there exists a 
$C<\infty$ such that
\begin{equation*}
\PP\left\{\left\|V_KV_K^\T - V_AV_A^\T\right\|_{\op} \le C  \sqrt{\kappa{\errw\over \lambda_K}}
\wedge 1 \right\} = 1-\cO(n^{-c}).
\end{equation*}
\end{lemma}

%


The following lemma provides an upper bound on the operator norm of $\bG H \bG^\T$ where  $\bG\in \RR^{n\times d}$ is a random matrix and its rows are independent sub-Gaussian random vectors. It is proved in Lemma 22 of \cite{bing2020prediction}.
\begin{lemma}\label{lem_op_norm}
Let $\bG$ be a $n\times d$ matrix with rows that are independent $\sigma$ sub-Gaussian  random vectors with identity covariance matrix. Then for all symmetric positive semi-definite matrices $H$, 
\[
\PP\left\{{1\over n}\| \bG H \bG^\T \|_{{\rm op}} \le \sigma^2\left( \sqrt{{\rm tr}(H) \over n} + \sqrt{6\|H\|_{\op}}
\right)^2\right\} \ge  1 -  \exp({-n}).
\]
\end{lemma}

\bigskip



\subsection*{Acknowledgements} Wegkamp is supported in part by by NSF grants  DMS 2210557 and  DMS 2015195.
\\		

{\small 
\setlength{\bibsep}{0.85pt}{
\bibliographystyle{ims}
\bibliography{ref}
}}

\end{document}